\def\year{2020}\relax
\documentclass[letterpaper]{article} 
\usepackage{aaai20}  
\usepackage{times}  
\usepackage{helvet} 
\usepackage{courier}  
\usepackage[hyphens]{url}  
\usepackage{graphicx} 
\usepackage{graphicx}  
\usepackage{amsmath}
\usepackage{amsfonts}
\usepackage{cleveref}
\usepackage{subfigure}
\usepackage{bbm}
\usepackage{algorithm}
\usepackage{algorithmic}

\usepackage{amsthm}
\usepackage{xpatch}
\xpatchcmd{\proof}{.}{\proofpunctuation}{}{}
\xpatchcmd{\proof}{\itshape}{\prooffont}{}{}
\newtheorem{theorem}{Theorem}
\newtheorem{corollary}{Corollary}
\newtheorem{lemma}{Lemma}
\newtheorem{assumption}{Assumption}
\newcommand{\proofpunctuation}{:}
\newcommand{\prooffont}{\bfseries}

\DeclareMathOperator*{\argmax}{argmax}

\newcommand{\prob}[1]{\mathbb{P} \left( #1 \right)}

\newcommand{\defeq}{\mathrel{\mathop:}=}

\newcommand{\indicator}[1]{\mathbbm{1} \left\{ #1 \right\}}

\newcommand{\pquantile}[2]{\theta_{#1} \left( #2 \right)}

\newcommand{\ceiling}[1]{\left \lceil #1 \right \rceil}

\newcommand{\flooring}[1]{\left \lfloor #1 \right \rfloor }

\newcommand{\bigparentheses}[1] {\left( #1 \right)}
 
\newcommand{\medthta}[1]{\theta_{\frac{1}{2}} \left( #1 \right)}

\newcommand{\samplesetcom}[3]{ \{ #1 \} _{#2}^{#3}}

\newcommand{\expectation}[1]{\mathbb{E}\left[ #1 \right]}
\begin{document}
 \pdfinfo{
/Title (Robust Stochastic Bandit Algorithms under Probabilistic Unbounded Adversarial Attack)
/Author (Ziwei Guan, Kaiyi Ji, Donald J. Bucci Jr., Timothy Y. Hu,Joseph Palombo, Michael Liston, Yingbin Liang)
} 

\setcounter{secnumdepth}{2} 
\setlength\titlebox{2.5in} 
\title{Robust Stochastic Bandit Algorithms under Probabilistic \\Unbounded Adversarial Attack}

\author{
	Ziwei Guan\textsuperscript{\rm 1}, Kaiyi Ji\textsuperscript{\rm 1},  Donald J. Bucci Jr. \textsuperscript{\rm 2}, Timothy Y. Hu \textsuperscript{\rm 2},\\
	{\bf \Large  Joseph Palombo\textsuperscript{\rm 2}, Michael Liston\textsuperscript{\rm 2},Yingbin Liang\textsuperscript{\rm 1} }\\
\textsuperscript{\rm1} The Ohio State University, ECE Department\\
2015 Neil Ave, Columbus, OH 43210\\
\textsuperscript{\rm2} 	Lockheed Martin Advanced Technology Laboratories\\
Cherry Hill, NJ, 08002, USA\\
}

\maketitle

\begin{abstract}
The multi-armed bandit formalism has been extensively studied under various attack models, in which an adversary can modify the reward revealed to the player. Previous studies focused on scenarios where the attack value either is bounded at each round or has a vanishing probability of occurrence. These models do not capture powerful adversaries that can catastrophically perturb the revealed reward. This paper investigates the attack model where an adversary attacks with a certain probability at each round, and its attack value can be arbitrary and unbounded if it attacks. Furthermore, the attack value does not necessarily follow a statistical distribution. We propose a novel sample median-based and exploration-aided UCB algorithm (called med-E-UCB) and a median-based $\epsilon$-greedy algorithm (called med-$\epsilon$-greedy). Both of these algorithms are provably robust to the aforementioned attack model. More specifically we show that both algorithms achieve $\mathcal{O}(\log T)$ pseudo-regret (i.e., the optimal regret without attacks). We also provide a high probability guarantee of $\mathcal{O}(\log T)$ regret with respect to random rewards and random occurrence of attacks. These bounds are achieved under arbitrary and unbounded reward perturbation as long as the attack probability does not exceed a certain constant threshold. We provide multiple synthetic simulations of the proposed algorithms to verify these claims and showcase the inability of existing techniques to achieve sublinear regret. We also provide experimental results of the algorithm operating in a cognitive radio setting using multiple software-defined radios.
\end{abstract}

\section{Introduction}
Stochastic multi-armed bandit models capture the scenarios where a player devises a strategy in order to access the optimal arm as often as possible. Such models have been used in a broad range of applications including news article recommendation~\cite{li2010contextual}, online advertising~\cite{pandey2007bandits}, medical treatment allocation~\cite{kuleshov2014algorithms}, and adaptive packet routing~\cite{awerbuch2004adaptive}. As security concerns have a critical impact in these applications, stochastic multi-armed bandit models under adversarial attacks have attracted  extensive attention. A variety of attack models have been studied under the multi-armed bandit formalism. Below we briefly summarize major models that are relevant to our study.

\begin{list}{-}{\topsep=0.ex \leftmargin=3mm \rightmargin=0.in \itemsep =-0.035in}
\item The adversarial multi-armed bandit model, in which an adversary is allowed to attack in each round with each attack subject to a bounded value. \cite{auer2002nonstochastic} proposed a robust EXP3 algorithm as a defense algorithm and \cite{audibert2009minimax,stoltz2005incomplete,bubeck2012regret} further provided tighter bounds. \cite{jun2018adversarial,liu2019data} showed that a small total attack cost of $\mathcal{O}(\log T)$ makes UCB and $\epsilon$-greedy algorithms fail with regret $\mathcal{O}(T)$.

\item The budget-bounded attack model, in which an adversary has a total budget of attack value but can choose to attack only over some time instances. The aim of defense is to achieve a regret that gradually transits between adversarial (the above always attack) and stochastic (never attack) models. \cite{lykouris2018stochastic} provided a variety of such robust algorithms and \cite{DBLP:journals/corr/abs-1902-08647} further developed an algorithm that improved the regret in \cite{lykouris2018stochastic}. 

\item The fractional attack model, in which the total rounds that an adversary attacks is limited either by probability that the adversary can attack or by the ratio of attacked rounds to total rounds. The attack value at each round is also subject to a bounded value. \cite{kapoor2019corruption} proposed a robust RUCB-MAB algorithm, which uses sample median to replace sample mean in UCB algorithm. \cite{seldin2014one} proposed an EXP3-based algorithm which can achieve optimality in both stochastic and adversarial cases.

\item The heavy-tail outlier model, in which the observed reward can have heavy-tail values, whose distribution has bounded first moment and unbounded second moment. \cite{bubeck2013bandits} proposed a robust Cantoni UCB algorithm that can defend against such heavy-tail outliers.
\end{list}

We observe that all of the above adversarial models assume that the attack value (i.e., the adversarial cost) either is bounded or has vanishing probability of occurrence. In this paper, we study an adversarial attack model where the attack value can be {\em arbitrary} and {\em unbounded}. To elaborate further, {\em an arbitrary attack value} allows flexible and adaptive attack strategies, which may not follow a probabilistic distribution. {\em Unboundedness} allows arbitrarily large attack values to occur with constant probability. Under such an attack model, it is impossible to defend if the adversary can attack at each round. Thus, we assume that the adversary attack with a fixed probability $\rho$ at each round (as justified in \cite{kapoor2019corruption,altschuler2018best}). 

Such an attack model turns out to be quite challenging due to arbitrary and unbounded attack values. As we demonstrate in Section \ref{sec:exp}, the existing popular (robust) algorithms fail to defend, even when attack values are not substantially large all the time. These algorithms include (a) vanilla UCB and $\epsilon$-greedy, which are mean-based and clearly are vulnerable under arbitrarily large attack; (b) EXP3, designed to succeed typically under bounded attack values; (c) Cantoni UCB for heavy-tail bandit \cite{bubeck2013bandits}, which requires large valued outliers to occur with asymptotically small probability; (d) RUCB-MAB \cite{kapoor2019corruption}, also designed to succeed typically under bounded attack values.

The contribution of this paper lies in proposing two novel robust median-based bandit algorithms to defend from arbitrary and unbounded attack values, and furthermore developing sharp bounds of their regret performance. 	
\subsection{Our Contributions}
We summarize our contributions as follows. 
\begin{list}{$\bullet$}{\topsep=0.1ex \leftmargin=0.15in \rightmargin=0.in \itemsep =-0.02in}
\item We propose a novel {\em \bf median-based exploration-aided UCB algorithm (med-E-UCB)} by incorporating a diminishing number of periodic exploration rounds. In contrast to RUCB-MAB  in \cite{kapoor2019corruption} (which directly replaces sample mean in vanilla UCB by sample median), our med-E-UCB adds a small amount of exploration, which turns out to be critical for maintaining logarithmic regret. We further propose a {\bf median-based $\epsilon$-greedy algorithm (med-$\epsilon$-greedy)}, for which logarithmic regret is achieved without additional exploration.  
\item For both med-E-UCB and med-$\epsilon$-greedy, we show that, even under arbitrary and unbounded attacks, they achieve an optimal $\mathcal{O}(\log T)$ pseudo-regret bound for no attack scenarios, as long as the attack probability $\rho$ does not exceed a certain constant threshold. This allows the number of attacks to scale linearly. We also provide a high-probability analysis with respect to both the randomness of reward samples and the randomness of attacks, so that $\mathcal{O} ( \log T )$ regret is guaranteed under almost all trajectory.

\item We develop a new analysis mechanism to deal with the technical challenge of incorporating the sample median into the analysis of bandit problems.  Direct use of the existing concentration bounds on the sample median does not provide a guarantee of $\mathcal{O} ( \log T )$ regret for our algorithms. In fact, it turns out to be nontrivial to maintain the concentration of sample median (which requires sufficient exploration of each arm) and at the same time best control the exploration to keep the scaling of the regret at the $\mathcal{O}(\log T)$ level. Such an analysis provides insight for us to design exploration procedure in med-E-UCB.

\item  We provide the synthetic demonstrations and experimental radio results from a realistic cognitive radio setting that demonstrate the robustness of the proposed algorithms and verify their $\mathcal{O}(\log T)$ regret under large valued attacks. We demonstrate in both sets of experiments that existing algorithms fail to achieve logarithmic regret under the proposed attack model.
\end{list}

\noindent All the technical proofs of the theorems in the paper can be found in the full version of this work posted on arXiv.
	
\subsection{Related Works}	
	\textbf{Stochastic vs adversarial multi-armed bandit.} Under an adversarial bandit model where attacks occur at each round, \cite{jun2018adversarial,liu2019data} showed that UCB and $\epsilon$-greedy fail with regret of $\mathcal{O}(T)$ while sustaining only a small total attack cost of $\mathcal{O}(\log T)$. Then, \cite{bubeck2012best,seldin2014one,auer2016algorithm,seldin2017improved,zimmert2018optimal,DBLP:journals/corr/abs-1902-08647} designed robust algorithms that achieve $\mathcal{O}(\sqrt{T})$ regret for an adversarial bandit and $\mathcal{O}(\log T)$ for stochastic bandit without attack. Furthermore, \cite{lykouris2018stochastic,DBLP:journals/corr/abs-1902-08647} provided robust algorithms and characterized the regret under the model where the fraction of attacks ranging from always-attacking to never-attacking. \cite{kapoor2019corruption} proposed a median-based UCB algorithm, and derived the same type of regret for a similar but probabilistic model, where the adversary attacks at each round with a certain probability. Other similar intermediate models were also studied in \cite{zimmert2018optimal}, where either the ratio of attacked rounds to total rounds or the value of attacks are constrained. All these studies assume a bounded range for the attack value at each round, whereas our study allows arbitrary and unbounded attack values.
	
\cite{bubeck2013bandits} proposed Cantoni UCB algorithm for a heavy-tail bandit, and showed that the algorithm can tolerate outlier samples. Though their heavy-tail distributions allow outliers to occur with fairly high probability as compared to sub-Gaussian distributions, our adversarial model is much more catastrophic. It allows the attack distribution to have an unbounded mean, whereas the heavy-tail distribution still requires a finite mean and certain higher order moments. For example, under our attack model, an adversary can attack only the optimal arm with probability $\rho$ by subtracting a sufficiently large constant $c$, so that the optimal arm no longer has the largest sample mean. Consequently, Cantoni UCB fails with the regret increasing linearly with $T$. The experiment in Section \ref{sec:exp} also shows that Cantoni UCB fails under our attack model. 

\textbf{Median-based robust algorithms.} The sample median is well known to be more robust than the sample mean in statistics~\cite{tyler2008robust,zhang2016provable}. Hence, the sample median has been used in a variety of contexts to design robust algorithms in multi-armed bandit problems \cite{altschuler2018best}, parameter recovery in phase retrieval \cite{zhang2016median}, and regression \cite{pmlr-v75-klivans18a}. In this paper, the analysis methods are novel and we provide high probability guarantees of $\mathcal{O}(\log T)$ regret (rather than just in average).  We also note that the RUCB-MAB \cite{kapoor2019corruption} directly replaces the sample mean by the sample median in UCB, which as shown in \Cref{sec:exp} fails to defend against arbitrary and unbounded attack.

	\section{Problem Formulation}\label{section: notation}

  Consider a $K$-armed bandit, where each arm $i$ exhibits a stationary reward distribution with a cumulative probability distribution function (CDF) $F_i$ and mean $\mu_i$. Denote the arm with the maximum mean reward as $i^*$, and assume that it is unique. Let $\mu^* = \mu_{i^*}$ represent the maximum mean reward and $\Delta_i = \mu^* - \mu_i$ for all $i\neq i^*$. Throughout the paper, we assume that $\mu_i$ for $i =1,... ,K$ are fixed constants and do not scale with the number of pulls $T$.
	
	At each round $t$, the player can pull any arm, $i\in \{1,..., K\}$. Then the bandit generates a reward $\tilde{X}_{i,t}$ according to the distribution $F_i$.
	
	There exists an adversary, who decides to attack with probability $0<\rho<1$, independently of the history of arm pulls by the player. If the adversary attacks, it adds an attack value $\eta_{t}$ to the reward so that the player observes a perturbed reward $X_{i,t} = \tilde{X}_{i,t} +\eta_{t}$; If the adversary does not attack, the player observes a clean reward $X_{i,t} = \tilde{X}_{i,t}$. That is,
	\begin{align}
	X_{i,t} = \begin{cases} \tilde{X}_{i,t} +\eta_{t}, \quad & \text{with probability } \rho; \\
	\tilde{X}_{i,t}, \quad & \text{with probability } 1-\rho. \end{cases} \label{eq:bandit_prob}
	\end{align}
	
	We emphasize that in our attack model, with a constant probability $\rho>0$, the attack value $\eta_t$ can be arbitrarily large. Furthermore, the realizations of $\eta_t$ do not follow any statistical model, which is much more catastrophic than the typical heavy-tail distributions\cite{bubeck2013bandits}. Since attack values can be arbitrary, our attack model allows the adversary to {\em adaptively} design its attack strategy based on reward history and distributions, as long as it attacks with probability $\rho$.
	
	For a bandit algorithm, we define the {\em pseudo-regret} as
	\begin{align}
	\bar{R}_T=\mu^*T-\expectation{\sum_{t=1}^T \mu_{I_t}},
	\end{align}
	where $I_t$ denotes the index of the arm pulled at time $t$, and the expectation is over both stochastic rewards as well as the random occurrence of attacks. It measures how the reward obtained by the algorithm deviates from that received by the optimal strategy in expectation. Furthermore, in practical adversarial scenarios, it is of great interest to characterize the regret in  reward trajectory. Thus, we define the following stronger notion of the {\em regret}
	\begin{align}
	R_T=\mu^*T-\sum_{t=1}^T \mu_{I_t}.
	\end{align}
	 Here, $R_T$ is a random variable with respect to the random occurrence of attacks and reward values, and in general is a function of attack values.
	
	The goal of this paper is to design algorithms that minimize the pseudo-regret and more importantly minimize the regret with high probability. More importantly, the latter condition guarantees robust operation over almost all reward trajectories.
	
	\textbf{Notations:} 
	For a given cumulative distribution function (CDF) $F$, we define its generalized $p$-quantile (where $0<p<1$)  function as $\theta_p(F) = \inf \{x\in \mathbb{R}: F(x) \ge p \}$. For a sample sequence $\{x_i\}_{i=1}^m$, let $\hat{F}$ be its empirical distribution. Then let $\theta_p(\{x_i\}_{i=1}^m)$ be the $p$-quantile of  $\hat{F}$, i.e., $ \theta_p(\{x_i\}_{i=1}^m)= \theta_p(\hat{F})$. If $p=1/2$, we obtain the median of the sequence given by $\mathrm{med}(\{x_i\}_{i=1}^m):=\theta_{1/2}(\{x_i\}_{i=1}^m)= \theta_{1/2}(\hat{F})$.
	
	We use $\mathcal{N}(\mu, \sigma^2)$ to denote a Gaussian distribution with mean $\mu$ and variance $\sigma^2$, and use $\Phi(\cdot)$ to denote the CDF of the standard Gaussian distribution $\mathcal{N}(0,1)$.
	
	In this paper, $\ceiling{\cdot}$ denotes the nearest bigger integer, $\lfloor \cdot \rfloor$ denotes the nearest smaller integer, and $[K]$ represents the set $\{1,2, ..., K\}$. Furthermore, $y = \mathcal{O}(f(x))$ represents that there exists constants $M>0, \zeta>0$ such that $y\le Mf(x)$ for all $x\ge \zeta$.  And $\log(\cdot)$ denotes the natural logarithm with the base $e$. For a differentiable function $f$, we write its derivative as $f'$.
	
\section{Median-based and Exploration-aided UCB}

In this section, we first propose a median-based UCB algorithm, and then show that such an algorithm can defend against the attack model described in Section \ref{section: notation}.

\subsection{Algorithm Overview}

We begin by explaining why direct replacement of sample mean by sample median in UCB \cite{kapoor2019corruption} cannot defend against large attack values. Consider a catastrophic attack scheme that our attack model allows, where the adversary sets $\eta_t = -\infty$ or a significantly large negative value in the case when the player pulls the optimal arm, and $\eta =0$ otherwise. Then, the first time that the player pulls the optimal arm, the adversary attacks with a positive probability $\rho>0$, resulting in the value of $\mathrm{med}_{i^*}(t) + \sqrt{\frac{\omega\log t}{T_j(t)}}$ being $-\infty$, where $\mathrm{med}_{i^*}(t)$ denotes the sample median of the rewards received by arm $i^*$ up to time $t$. Consequently, the optimal arm will never be pulled in the future, and hence the regret grows linearly with $T$. The primary reason that median-based vanilla UCB fails in such a case is due to insufficiently enforced exploration for each arm. That is, the sample median can fail with only one catastrophic sample if there are not enough samples. On the other hand, if there are further enforced explorations to pull the optimal arm, the sample median can eventually rule out outlier attack values since such an attack occurs only probabilistically and not all the time. 
\begin{figure}
	\centering
	\includegraphics[width= 0.45\textwidth]{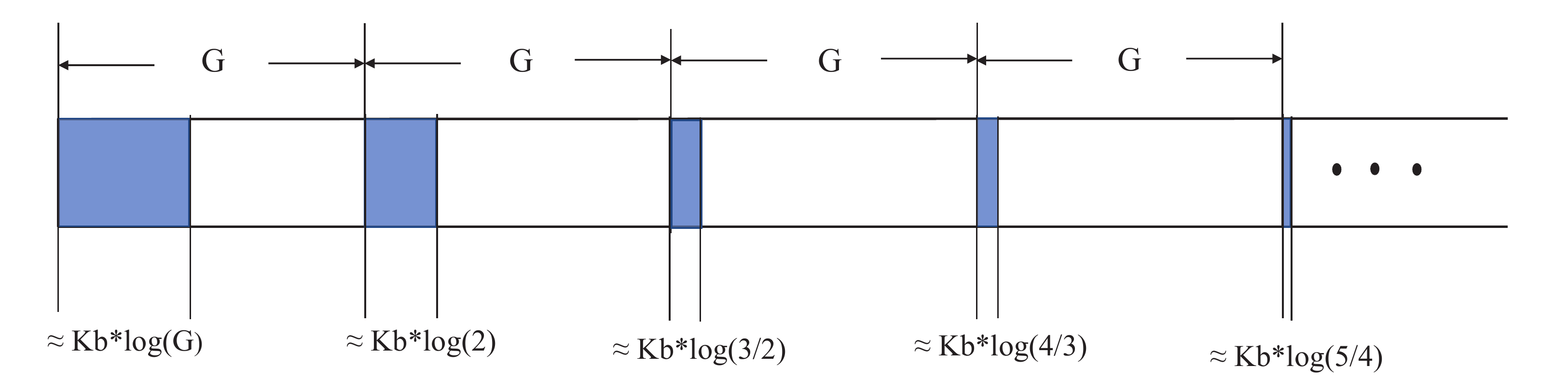}
	\caption{An illustration of med-E-UCB scheme, where the blue blocks denote the pure exploration rounds and the white blocks denote the UCB rounds.}\label{fig:illus}
\end{figure}
The above understanding motivates us to design an exploration-aided UCB algorithm, and then incorporate the sample median to defend against large attack values. We call this algorithm med-E-UCB and describe it formally in \Cref{alg:meucb}. This idea is illustrated in \Cref{fig:illus}, where we divide the pulling rounds into blocks. Each block consists of $G$ rounds, where $G\ge Kb\log G$ and $b>0$ is an arbitrary constant. During the first block (i.e. $k=0$), the size of pure exploration round is fixed at $b\log G$ pulls per arm. Except for the first block, at the beginning of each block, say block $k$, each arm is approximately explored $b\log \frac{k+1}{k}$ rounds. As a result, each arm is guaranteed to have been pulled $b\log((k+1)G)$ times at block $k$. So that pure exploration does not significantly affect the regret. 

\begin{algorithm}
	\caption{med-E-UCB}
	\label{alg:meucb}
	\textbf{Input:} Number of arms $K$, group size $G$, exploration parameters $b $,$\omega$, and total rounds $T$.
	\begin{algorithmic}[1]
		\STATE Initialization: for the first $K\ceiling{b \log G}$ rounds, pull each arm $\ceiling{b\log G}$ times.
		\FOR{$t = K\ceiling{b\log G }+1, ...,\  T$}
		\STATE $k = \left\lfloor\frac{t}{G} \right\rfloor$;
		\IF {$kG+1 \le t \le kG + K(\ceiling{b\log (k+1)G} - \ceiling{b\log kG})$}
		\STATE  Pure Exploration: \\ \quad $I_t = \ceiling{\frac{t-kG}{\ceiling{b\log (k+1)G} - \ceiling{b\log kG}}}$;
		\ELSE
		\STATE UCB round: \\
		\quad $I_t = \argmax\limits_j  \left\{ \mathrm{med}_j(t-1) + \sqrt{\frac{\omega \log t}{T_j(t-1)}}\right\}$, \\
		\ENDIF
		\ENDFOR
	\end{algorithmic}
\end{algorithm}

\subsection{Analysis of Regret}
In this subsection, we analyze the regret of med-E-UCB. The distributions associated with the arm are not necessarily Gaussian, and need only satisfy the following assumption.

\begin{assumption}\label{ass:mucb}
	There exists a constant $s$, such that $\pquantile{\frac{1}{2}-s}{F_{i^*}} > \pquantile{\frac{1}{2}+s}{F_j}$ for all $j\neq i^*$. Moreover, $F_i(\cdot)$ is differentiable for all $i\in [K]$ , and there exist constants $l>0$ and $\xi>0$, such that  
	$$\left. \inf\{ F'_{i^*}(x) : \ \pquantile{\frac{1}{2}-s}{F_{i^*}}-\xi< x < \pquantile{\frac{1}{2}-s}{F_{i^*}} \}\ge l\right.,$$
	and for all $j\neq i^*$
	$$\inf\{F'_{j}(x) : \ \pquantile{\frac{1}{2}+s}{F_j}< x <  \pquantile{\frac{1}{2}+s}{F_{j}}+\xi\}\ge l.$$ 	
\end{assumption}

The above assumption essentially requires that the median of the optimal arm and the median of the non-optimal arms have gaps such that the optimal arm can stand out statistically. This assumption further requires that the probability density within a $\xi$-neighborhood of the median to be lower-bounded by a positive $l$ in order to guarantee a good concentration property. Clearly, Gaussian distributions satisfy Assumption~\ref{ass:mucb}.
	\begin{lemma}[Sample median concentration bound]
		\label{median_concerntration}
		Let $X_i  = \tilde{X}_i + \eta_i, i=1,....,n$ be $n$ attacked data samples, where $\tilde{X}_i, i=1,...,n$ are original (i.e., un-attacked)  data samples  i.i.d. drawn from the distribution with CDF $F$. The $\eta_i$'s are unbounded attack values. If $\sum_{i=1}^n \indicator{\eta_i \neq 0} \le s\cdot n$ holds for a constant $s\in (0,1)$, then for any $a,b > 0$, we  have
		\begin{align*}
		&\prob{\mathrm{med}\bigparentheses{\samplesetcom{X_i}{i=1}{n}}- \pquantile{\frac{1}{2}-s}{F} \le -a}\le \exp\bigparentheses{-2np_1^2},\\
		&\prob{\mathrm{med}\bigparentheses{\samplesetcom{X_i}{i=1}{n}}- \pquantile{\frac{1}{2}+s}{F} \ge b}\le \exp\bigparentheses{-2np_2^2},
		\end{align*}
		where $p_1 =\frac{1}{2}-s - F(\pquantile{\frac{1}{2}-s}{F}-a)$, and $p_2 =F(\pquantile{\frac{1}{2}+s}{F}+b)- \frac{1}{2} -s$.
	\end{lemma}
Using \Cref{median_concerntration}, we obtain the following regret bounds for med-E-UCB. 
\begin{theorem}\label{th:avg_mucb}
	Consider the stochastic multi-armed bandit problem as described in \eqref{eq:bandit_prob}. Suppose Assumption~\ref{ass:mucb} holds. Further assume that the attack probability $\rho < s$, total number of rounds $T > G$, $b \ge \max\{\frac{\omega}{\xi^2}, \frac{2}{(s-\rho)^2}\}$ and $\omega \ge \frac{2}{l^2}$. Then the pseudo-regret of med-E-UCB satisfies
	\begin{align*}
	\bar{R}_T &\le\sum_{j=1, j\neq i^*}^{K}\Delta_j b\log(2T)\\
	&\quad+ \sum_{j=1, j\neq i^*}^{K}\Delta_j \left(\frac{4\omega\log T}{(\pquantile{\frac{1}{2} -s}{F_{i^*}} - \pquantile{\frac{1}{2} + s}{F_j})^2} \right) \\
	&\quad+ \sum_{j=1, j\neq i^*}^{K}\Delta_j(2 + \frac{2\pi^2}{3}).
	\end{align*}	
	For constant $K$, $\Delta_j$, $b$ and $\omega$, $\bar{R}_T = \mathcal{O}(\log(T))$.
\end{theorem}

\Cref{th:avg_mucb} implies that med-E-UCB achieves the best possible pseudo-regret bound under the attack-free model \cite{lai1985asymptotically}. Considering that the number of attacks can scale linearly with the total number of pulling rounds and attack values can be unbounded and arbitrary, \Cref{th:avg_mucb} demonstrates that med-E-UCB is robust algorithm against very powerful attacks.

Furthermore, we establish a stronger {\em high-probability} guarantee for the regret of med-E-UCB with respect to both the randomness of attack occurrence and rewards. 

\begin{theorem}\label{th:highprob_meu}
	Suppose Assumption \ref{ass:mucb} holds. Assume that the attack probability $\rho < s$, total number of rounds $T >G$, $b \ge \max\{\frac{\omega}{\xi^2}, \frac{2}{(s-\rho)^2}\}$ and $\omega \ge \frac{3.5}{l^2}$. Then, with probability at least $1-\delta$ with respect to the randomness of attack occurrence and rewards, the regret of med-E-UCB satisfies 
	\begin{align*}
	&R_T \le \sum_{j=1, j\neq i^*}\Delta_j b\log(2T)\\
	&\quad + \sum_{j=1, j\neq i^*}\Delta_j\bigparentheses{\frac{4\omega\log T}{(\pquantile{\frac{1}{2} -s}{F_{i^*}} - \pquantile{\frac{1}{2} + s}{F_j})^2}}\\
	&\quad+ \sum_{j=1, j\neq i^*} \Delta_j\bigparentheses{e\bigparentheses{\frac{bK}{2\delta}}^\frac{1}{4} + \frac{2K}{\delta} +3 + \frac{\pi^2}{3}}.
	\end{align*}
	For constant $K$,$\Delta_j$, $b$ and $\omega$, $R_T =\mathcal{O} (\log T + \frac{1}{\delta})$.
\end{theorem}

In practice, the high-probability result as Theorem \ref{th:highprob_meu} is much more desirable. In such a case, we would like the guarantee of successful defense for almost all realizations of the attack (i.e., with high probability) rather than an on-average performance which does not imply what happens for each attack realization.

Theorems \ref{th:avg_mucb} and \ref{th:highprob_meu} readily imply the following corollary for Gaussian distributions. To present the result, let the $i$th arm be associated with $\mathcal{N}(\mu_i, \sigma^2)$. Further let $\Delta_{min}:= \min\limits_{i\neq i^*} \{\mu^* - \mu_i\}$, and $l = \frac{1}{\sqrt{2\pi\sigma^2}}\exp\bigparentheses{-\frac{(\Delta_{min}+4)^2}{32\sigma^2} }$, where $\Phi(\cdot)$ denotes the CDF of $\mathcal{N}(0, 1)$.

\begin{corollary}\label{cor:mucb}
	Suppose each arm corresponds to a Gaussian distribution. Suppose $\rho<\Phi(\frac{\Delta_{min}}{4\sigma}) -\frac{1}{2}$, $b \ge \max\{\omega, \frac{2}{(\Phi(\Delta_{min}/(4\sigma)) -1/2 -\rho)^2}\}$, and $\omega \ge \frac{2}{l^2}$. Then, the pseudo-regret of med-E-UCB satisfies $\bar{R}_T =  \mathcal{O}(\log T)$.
	And, with probability at least $1-\delta$ with respect to the randomness of both attacks and rewards, the regret of med-E-UCB satisfies $R_T = \mathcal{O} \left(\log (T) +\frac{1}{\delta}\right)$.
\end{corollary}

\section{Median-based $\epsilon$-greedy}

In this section, we propose a robust $\epsilon$-greedy algorithm based on the sample median, and show that it is robust to defend against the adversarial attack described in Section \ref{section: notation}. This algorithm is helpful to compare with med-E-UCB to illustrate that under unbounded attacks med-E-UCB is more exploration-efficient.

\subsection{Algorithm Overview}

We propose an $\epsilon$-greedy algorithm that incorporates the sample median to defend against adversarial attacks. We call the algorithm med-$\epsilon$-greedy and describe it formerly in Algorithm \ref{alg: meg}. Compared with the standard $\epsilon$-greedy algorithm, the med-$\epsilon$-greedy algorithm replaces the sample mean by the sample median. In addition, the exploration parameter $c$ needs to be appropriately chosen to provide sufficient exploration to guarantee the concentration of the sample median.

\begin{algorithm}
	\caption{med-$\epsilon$-greedy}\label{alg: meg}
	\textbf{Input:} Number $K$ of arms, total number of $T$ pulling, and exploration parameter $c$.
	\begin{algorithmic}[1]
		\STATE Initialization: pull each arm $\ceiling{c}$ times.
		\FOR {$t= \lceil c\rceil K+1, ...,\  T$}
		\STATE Pull arm 
		$$
		I_t = \left\{ 
		\begin{aligned}
		&\argmax_{j} \{\mathrm{med}_j(t-1)\}, \qquad  \text{w.p.  } 1 - \frac{cK}{t}\\
		&\text{Uniformly pick an arm from 1 t o K}, \text{w.p.  } \frac{cK}{t}  
		\end{aligned} \right.$$
		\ENDFOR
	\end{algorithmic}
\end{algorithm}

\subsection{Analysis of Regret}\label{sec:reg_meg}
In this subsection, we analyze both the pseudo-regret and regret of the med-$\epsilon$-greedy algorithm. We first make the following assumption on the reward distributions.
\begin{assumption}\label{ass:meg}
	There exists a constant $0 < s <1$ and a constant $x_0\in \mathbb{R}$, such that the CDF $F(\cdot)$ of the optimal arm satisfies $F_{i^*}(x_0) < \frac{1}{2}-s$, and the CDFs of the remaining arms satisfy $F_{j}(x_0) > \frac{1}{2}+s$, for all $j\neq i^*$. 
\end{assumption}
The above assumption ensures that the sample median of the optimal arm is larger than those of the other arms with a desirable gap. Compared to Assumption \ref{ass:mucb} for med-E-UCB, Assumption \ref{ass:meg} is slightly weaker as it does not need the CDF to be differentiable and its derivative to be bounded below in the neighborhood of $\frac{1}{2}+s$ or $\frac{1}{2}-s$ quantiles. Clearly, a collection of Gaussian distributions with a unique largest mean satisfies Assumption \ref{ass:meg}.

The following theorem characterizes the pseudo-regret bound for the med-$\epsilon$-greedy algorithm.

\begin{theorem}\label{th: avg_meg}
	Consider the stochastic multi-armed bandit problem under adversarial attack as described in \eqref{eq:bandit_prob}. Let Assumption~\ref{ass:meg} hold. Suppose $\rho < s$, and suppose the exploration parameter $c$ satisfies the following condition 
	 $c > \max \{20, \frac{2}{\bigparentheses{F_j(x_0)-\frac{1}{2}  -s}^2}, \frac{2}{\bigparentheses{\frac{1}{2} -s - F_{i^*}(x_0)}^2}, \frac{2}{(s-\rho)^2} : j =1, 2, ...K, j \neq i^*\}$.  Then the pseudo-regret of med-$\epsilon$-greedy satisfies
	$$\bar{R}_T \le  c\sum_{j=1, j\neq i^*}^{K}\Delta_j\log T + 2cKe\mu^* + \sum_{j=1, j\neq i^*}^{K}(2+3c)\Delta_j,$$
	For fixed $\Delta_j$, $K$, and $c$, $\bar{R}_T = \mathcal{O}(\log T)$.
\end{theorem}

\Cref{th: avg_meg} indicates that even under adversarial attack, med-$\epsilon$-greedy still achieves $\mathcal{O}(\log T)$ regret, which is the same as the optimal pseudo-regret order in attack-free model. In contrast to med-E-UCB, exploration rounds in vanilla $\epsilon$-greedy are already sufficient for the sample median to be effective.

Aside from the pseudo-regret bound, we further provide a high-probability guarantee for the regret below.

\begin{theorem}\label{th:highprob_meg}
	Given Assumption \ref{ass:meg}, suppose $\rho < s$, and the exploration parameter $c$ satisfies the following condition 
	$c > \max \{40, \frac{4}{\bigparentheses{F_j(x_0)-\frac{1}{2}  -s}^2}, \frac{4}{\bigparentheses{\frac{1}{2} -s - F_{i^*}(x_0)}^2}, \frac{1}{(s-\rho)^2}: j =1, 2, ...K, j \neq i^*\}.$  Then, with probability at least $1-\delta$ with respect to the randomness of both attacks and rewards, the regret of med-$\epsilon$-greedy satisfies
	
	$$R_T\le\frac{6\ceiling{c}^2K^3}{\delta}\mu^* + \sum_{j=1, j\neq i^*}^K2c\Delta_j \log T + \sum_{j=1, j\neq i^*}^K 2c\Delta_j.$$
	For constant $\Delta_j$, $K$, and $c$, $R_T = \mathcal{O}(\log T + \frac{1}{\delta})$
\end{theorem}

Theorems \ref{th: avg_meg} and 4 readily implies the result when all arms correspond to Gaussian distributions, which we state in the following corollary. Similar to \Cref{cor:mucb} for med-E-UCB, for Gaussian distributions, we have derived the threshold for $\rho$ below which med-$\epsilon$-greedy has the desired regret. To present the result, suppose the $i$th arm is associated with $\mathcal{N}(\mu_i, \sigma^2)$, and let $\Delta_{min}:= \min\limits_{i\neq i^*} \{\mu^* - \mu_i\}$.
\begin{corollary}\label{cor: gau_meg}
	Suppose each arm corresponds to a Gaussian distribution, and $\rho<\Phi(\frac{\Delta_{min}}{4\sigma}) -\frac{1}{2}$. Let $c> \max\{10, \frac{1}{\left(\Phi(\frac{\Delta_{min}}{2\sigma}) - \Phi(\frac{\Delta_{min}}{4\sigma})\right)^2}, \frac{1}{(\Phi(\frac{\Delta_{min}}{4\sigma}) -\frac{1}{2 } -\rho)^2}\}.$ Then, the pseudo-regret of med-$\epsilon$-greedy satisfies $\bar{R}_T = \mathcal{O}(\log T )$.
	
	Furthermore, with probability at least $1-\delta$ with respect the randomness of both attacks and rewards, the regret of med-$\epsilon$-greedy satisfies
	$R_T = \mathcal{O}(\log T + \frac{1}{\delta})$.
\end{corollary}

\section{Experiments}\label{sec:exp}
\subsection{Comparison among Algorithms}
\label{sec:simulation}
In this subsection, we provide experiments to demonstrate that existing robust algorithms fail under the attack model considered here, whereas our two algorithms are successful. 

In our experiment, we choose  the number of arms to be $10$. The reward distribution of the  $i$th arm is $\mathcal{N}(2i, 1)$ for $i\in [K]$. The attack probability is fixed to be $\rho$ ($\rho = 0.125$ and $0.3$). The adversary generates an attack value $\eta$ uniformly at random from the interval $(0, 1800)$ if it attacks, and subtracts the clean reward  by $\eta$ if the optimal arm is pulled and adds to the clean reward otherwise. For each algorithm, each trial contains $T = 10^5$ rounds, and the final results take the average of $20$ Monte Carlo trials. 

We first compare the performance of our med-E-UCB and med-$\epsilon$-greedy with RUCB-MAB \cite{kapoor2019corruption}, EXP3 \cite{auer2002nonstochastic}, and Cantoni UCB \cite{bubeck2013bandits}. We also include (vanilla) UCB\cite{auer2002finite} and (vanilla) $\epsilon$-greedy \cite{auer2002finite} in the comparison for completeness. For med-E-UCB, we set $b= 4, \omega=4$, and $G= 10^3$. For med-$\epsilon$-greedy, we set $c = 10$. Other parameters are set as suggested by the original references. It can be seen from \Cref{Fig:regret_multialg} that our med-E-UCB and med-$\epsilon$-greedy algorithms significantly outperform the other algorithms with respect to the average regret. It is also clear that only our med-E-UCB and med-$\epsilon$-greedy algorithms have logarithmically increasing regret, whereas all other algorithms suffer linearly increasing regret. Between our two algorithms, med-E-UCB performs slightly better than med-$\epsilon$-greedy. This implies that the med-E-UCB is more exploration efficient than med-$\epsilon$-greedy. The same observations can also be made in \Cref{Fig:per_multialg}, where the performance metric is the percentage of pulling the optimal arm. Only our med-E-UCB and med-$\epsilon$-greedy algorithms asymptotically approaches 100\% optimal arm selection rate.

\begin{figure}[!htb]	
	\centering
	\subfigure[$\rho=0.125$]{
		\includegraphics[width=0.24\textwidth]{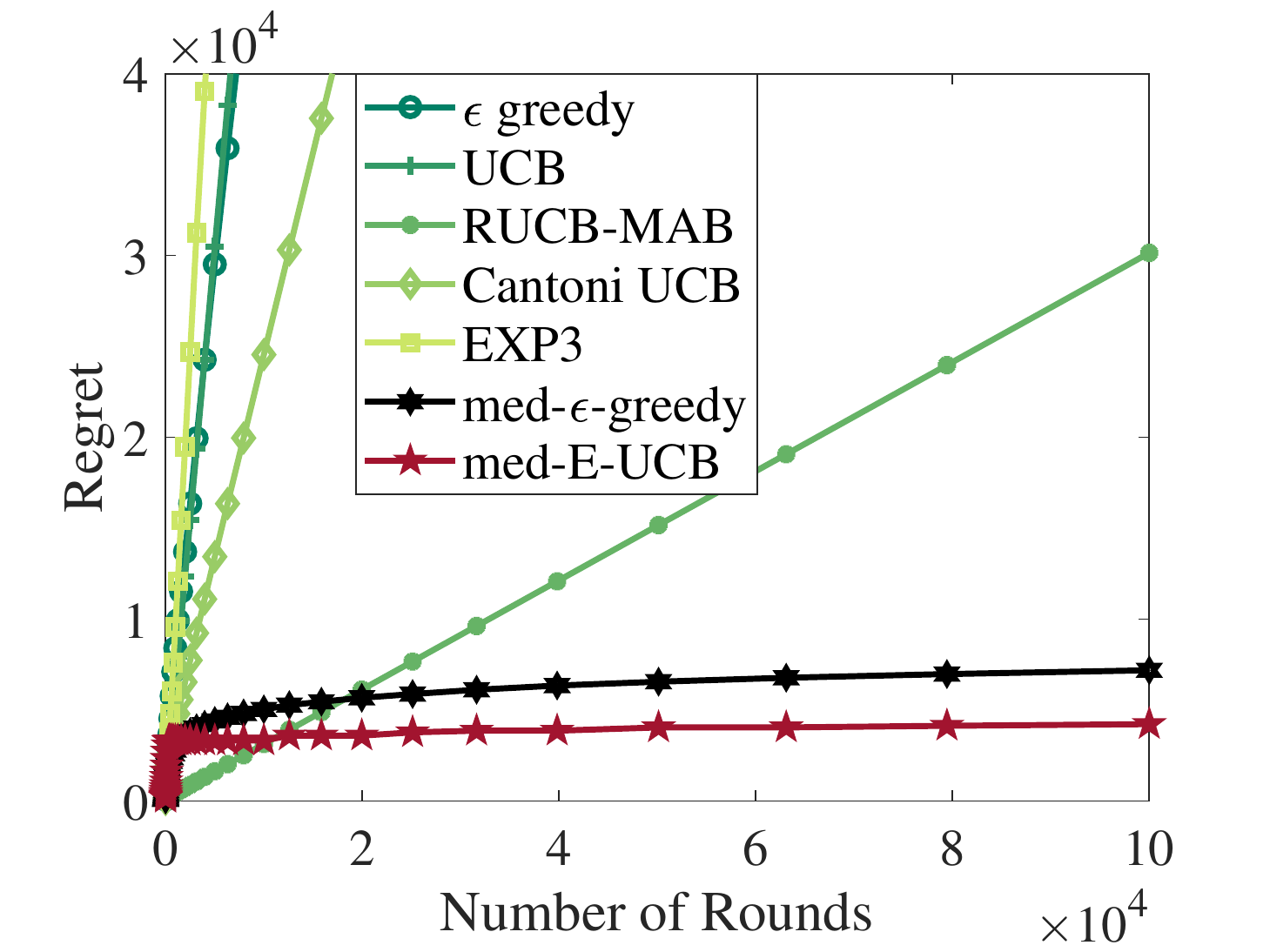}}\subfigure[$\rho=0.3$]{
		\includegraphics[width=0.24\textwidth]{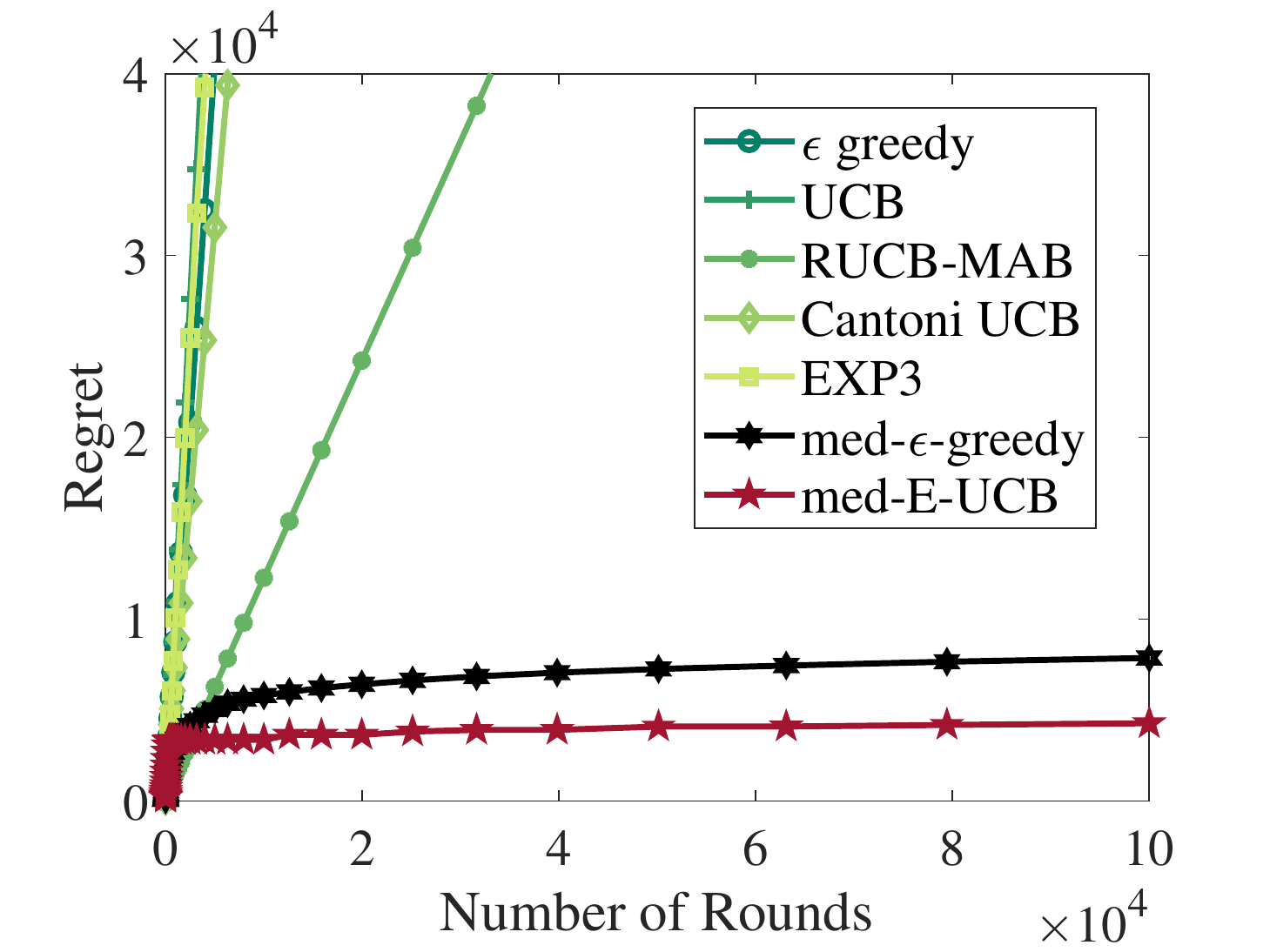}}
	\caption{Comparison of regret among algorithms}\label{Fig:regret_multialg}
\end{figure}

\begin{figure}[!htb]	
	\centering
	\subfigure[$\rho=0.125$]{ 
		\includegraphics[width=0.24\textwidth]{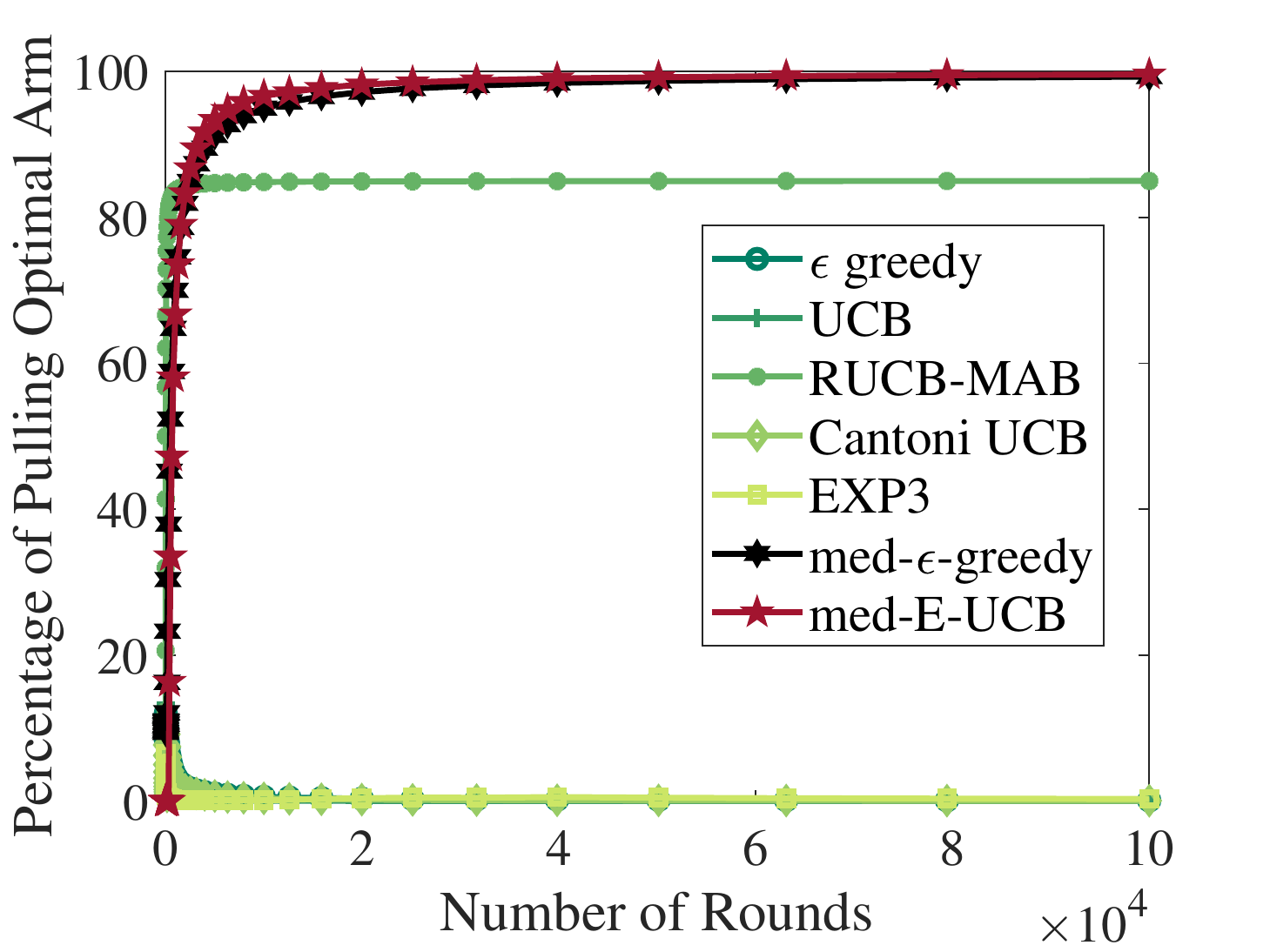}}\subfigure[$\rho=0.3$]{ 
		\includegraphics[width=0.24\textwidth]{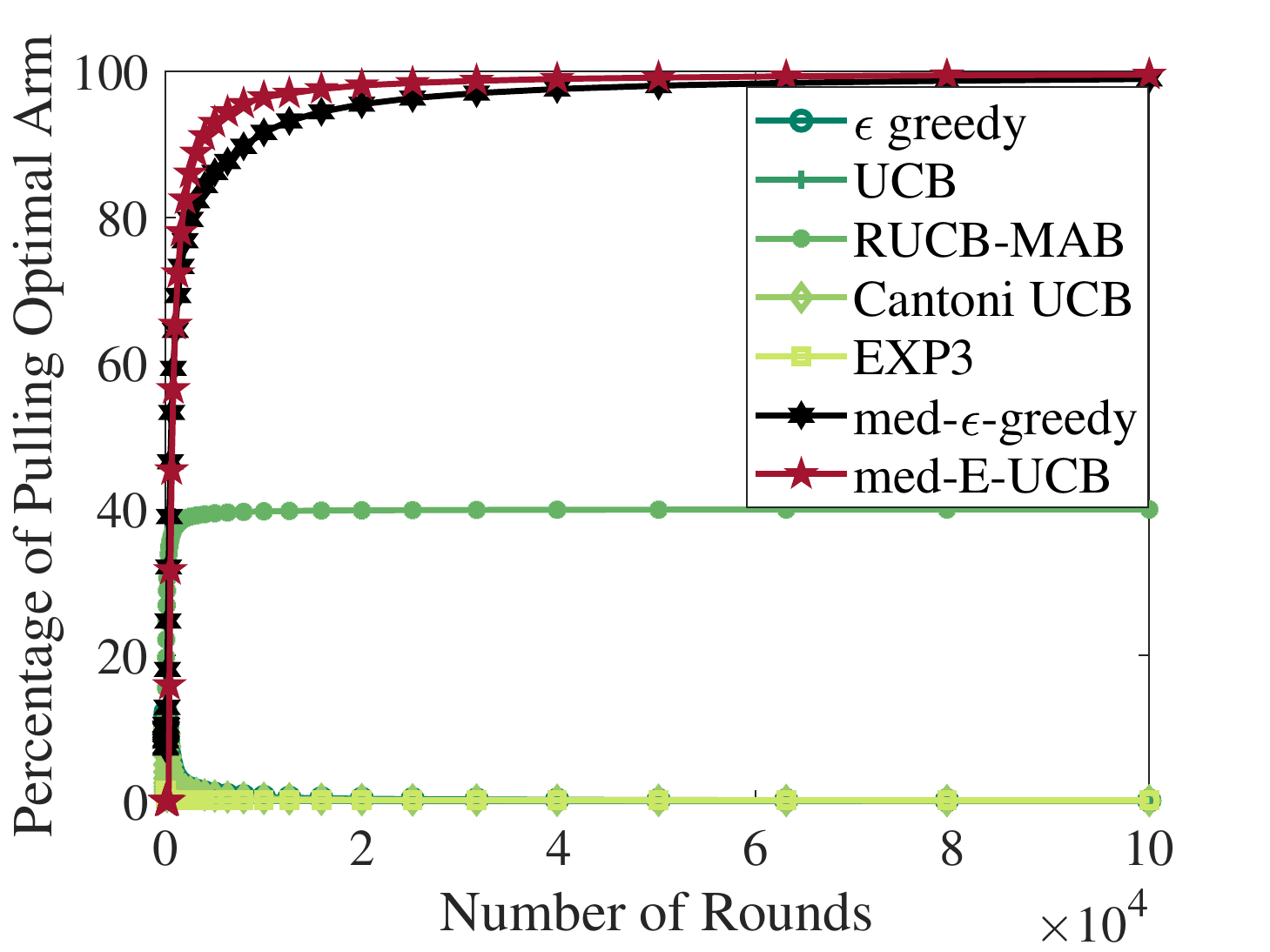}}
	\caption{Comparison of percentage of pulling the optimal arm among algorithms}\label{Fig:per_multialg}  
\end{figure}

\begin{figure}[!htb]	
	\subfigure[$\rho= 0.125$]{ 
		\includegraphics[width=0.24\textwidth]{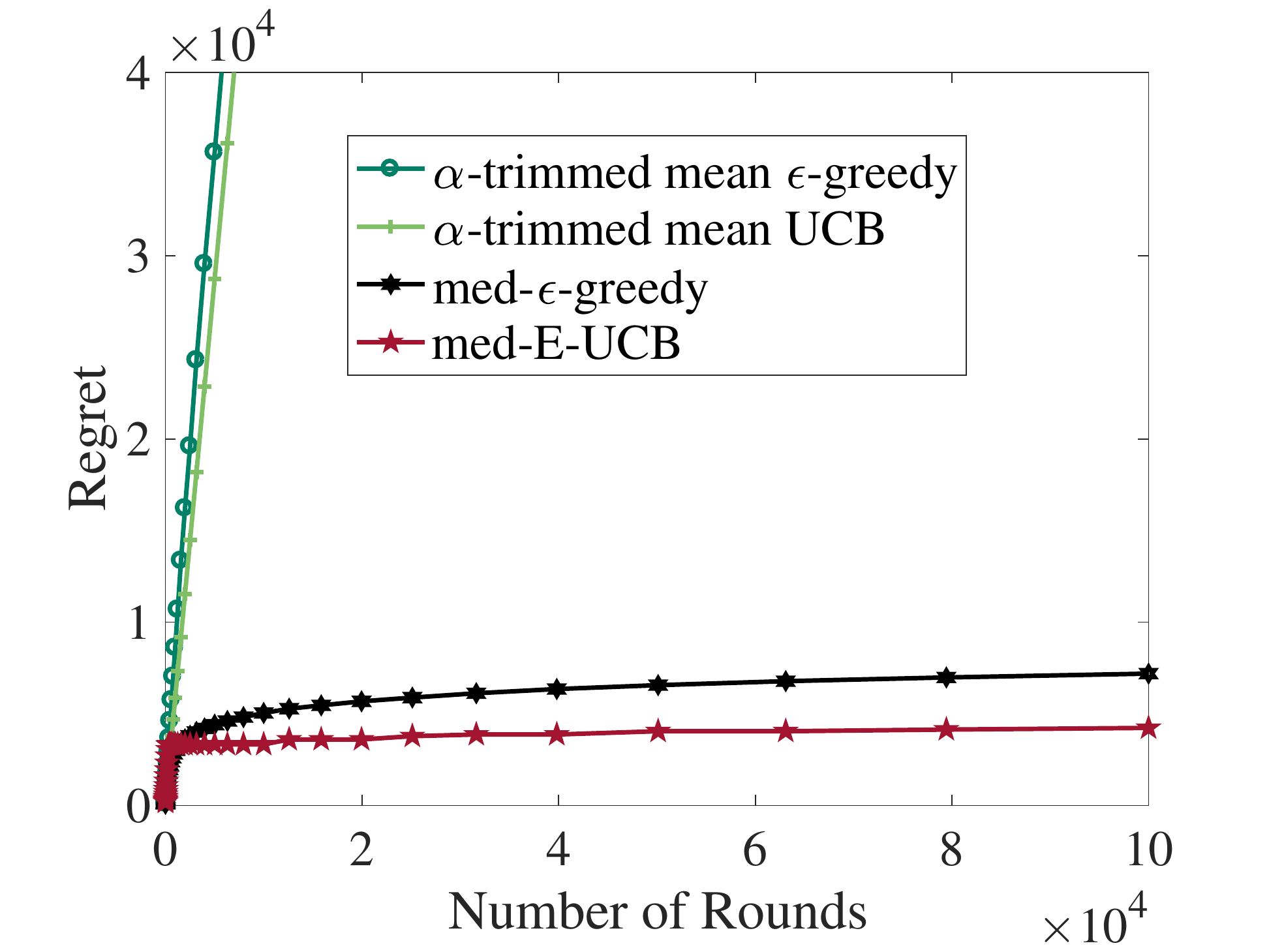}}\subfigure[$\rho= 0.3$]{ 
		\includegraphics[width=0.24\textwidth]{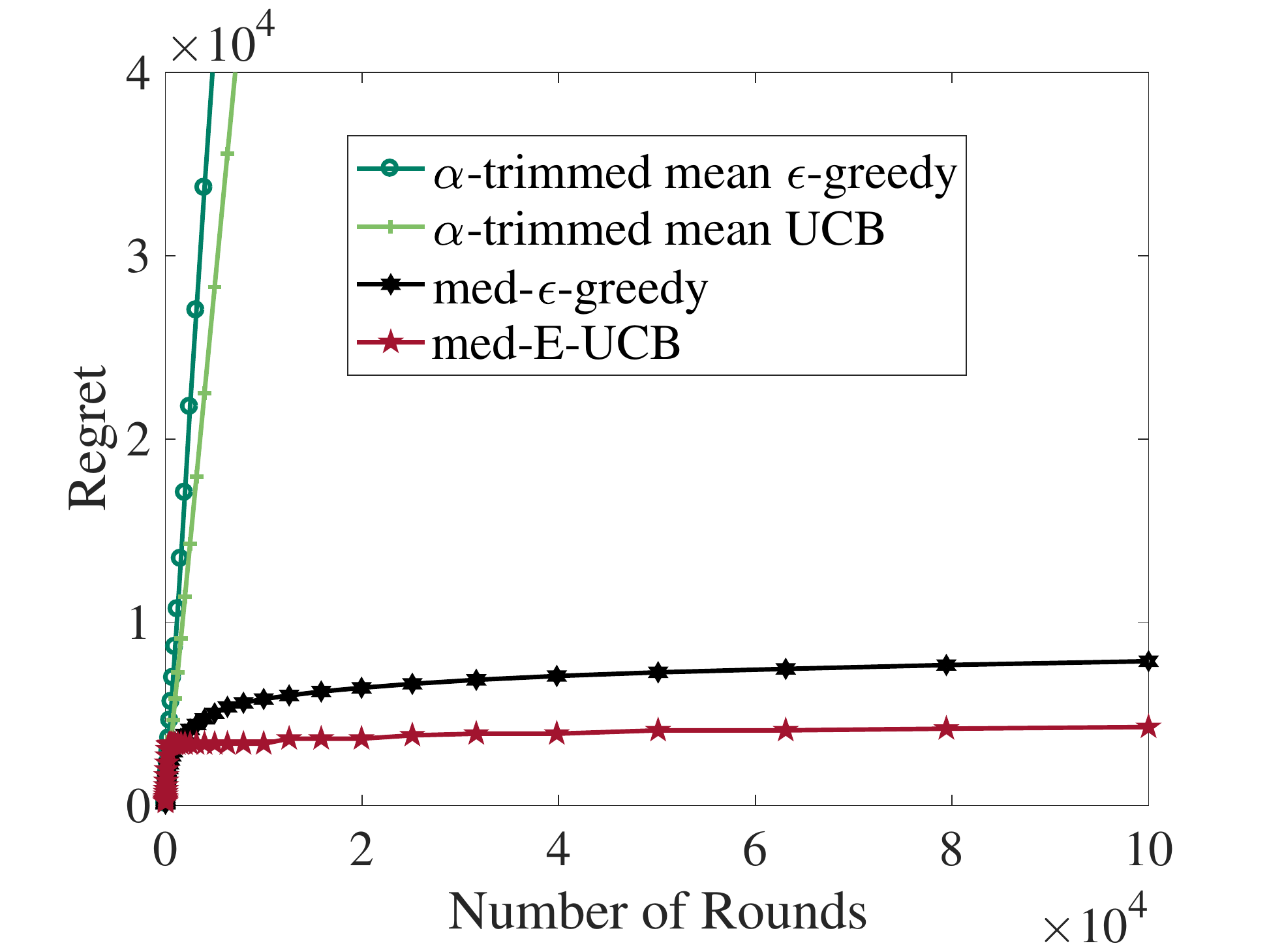}}
	\caption{Comparison of regret among algorithms}\label{Fig:regret_alpha}
\end{figure}
\begin{figure}[!htb]	
	\subfigure[$\rho= 0.125$]{ 
		\includegraphics[width=0.24\textwidth]{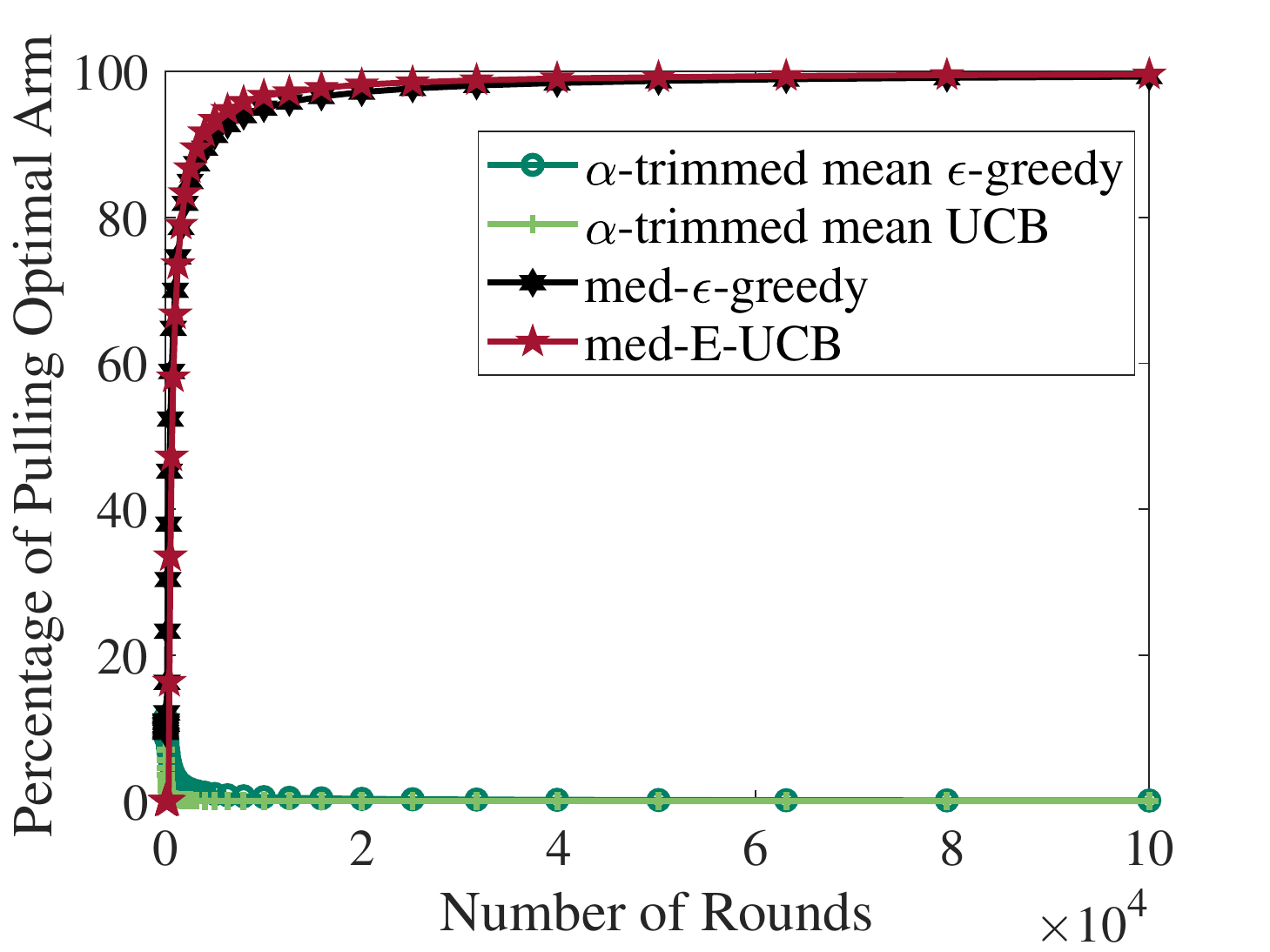}}\subfigure[$\rho= 0.3$]{ 
		\includegraphics[width=0.24\textwidth]{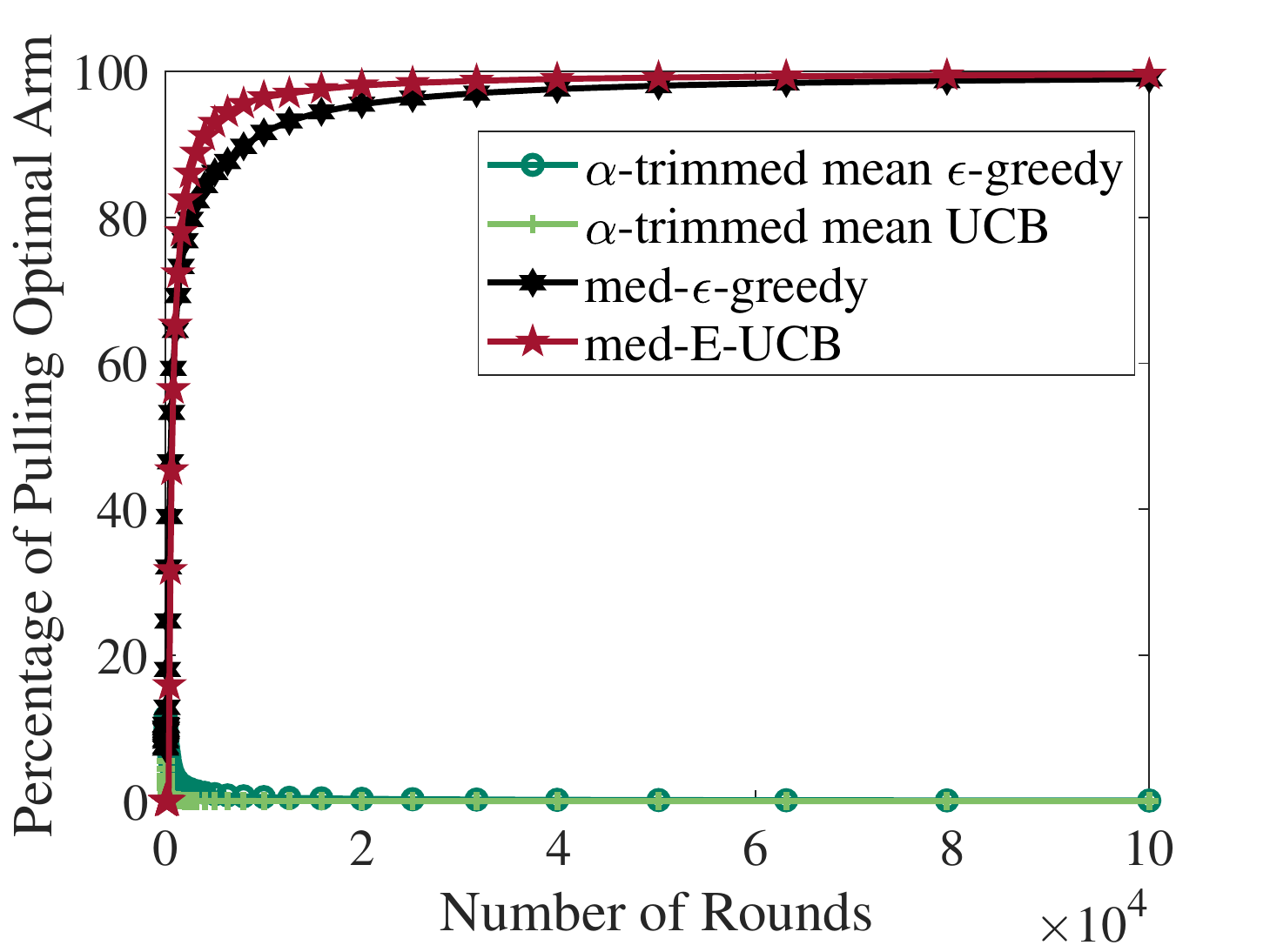}}
	\caption{Comparison of percentage pulling of optimal arm among algorithms}\label{Fig:per_alpha}
\end{figure}

We further note that as shown in \Cref{Fig:regret_multialg,Fig:per_multialg}, RUCB-MAB in \cite{kapoor2019corruption} does not defend against large valued attacks. This fact indicates that direct replacement of sample mean by sample median as in RUCB-MAB does not yield sufficiently robust performance. As a result, the diminishing periodic exploration in our med-E-UCB turns out to play a critical role in its successful defense.  

We further compare the performance of our med-E-UCB and med-$\epsilon$-greedy with a so-called $\alpha$-trimmed scheme \cite{bednar1984alpha}, which is a very popular method in signal processing to deal with similar unbounded arbitrary attacks. The idea of such a scheme is to remove the top and bottom  $\alpha$ fraction of samples before calculating the sample mean in order to eliminate the influence of outliers. It can be seen from \Cref{Fig:regret_alpha,Fig:per_alpha} that med-E-UCB and med-$\epsilon$-greedy significantly outperform the $\alpha$-trimmed UCB and $\alpha$-trimmed $\epsilon$-greedy algorithms.

\subsection{Experiment over Cognitive Radio Testbed}\label{sec:radio}
In this subsection, we present experimental results over a wireless over-the-air radio testbed (\Cref{fig:usrp_setup}) to validate the performance of med-E-UCB and med-$\epsilon$-greedy. The testbed models a pair of secondary users in a cognitive radio network opportunistically sharing spectrum resources with a primary user (not shown) while simultaneously defending against a stochastic adversarial jammer. We model the channel selection problem of the secondary users as a multi-armed bandit and use the channel signal to interference and noise ratio (SINR) as a measure of reward. The SINR is approximated using the inverse of error vector magnitude, which has a linear relationship for the range of signal powers we measure in the software defined radio (SDR). The transmitted packet signal power is constant throughout the experiment so that lower noise channels have a greater SINR compared with noisier channels which have a lower SINR. We model our unbounded adversary according to \eqref{eq:bandit_prob}. If the adversary attacks, a fixed power noise jamming attack is placed over top the secondary user signal packet, significantly reducing SINR by 40dB (i.e, 4 orders of magnitude). Each radio node uses an Ettus Research Universal Software Radio Peripheral (USRP) B200 Software Defined Radio (SDR).

\begin{figure}[!htb]
	\centering
		\includegraphics[width=0.60\linewidth]{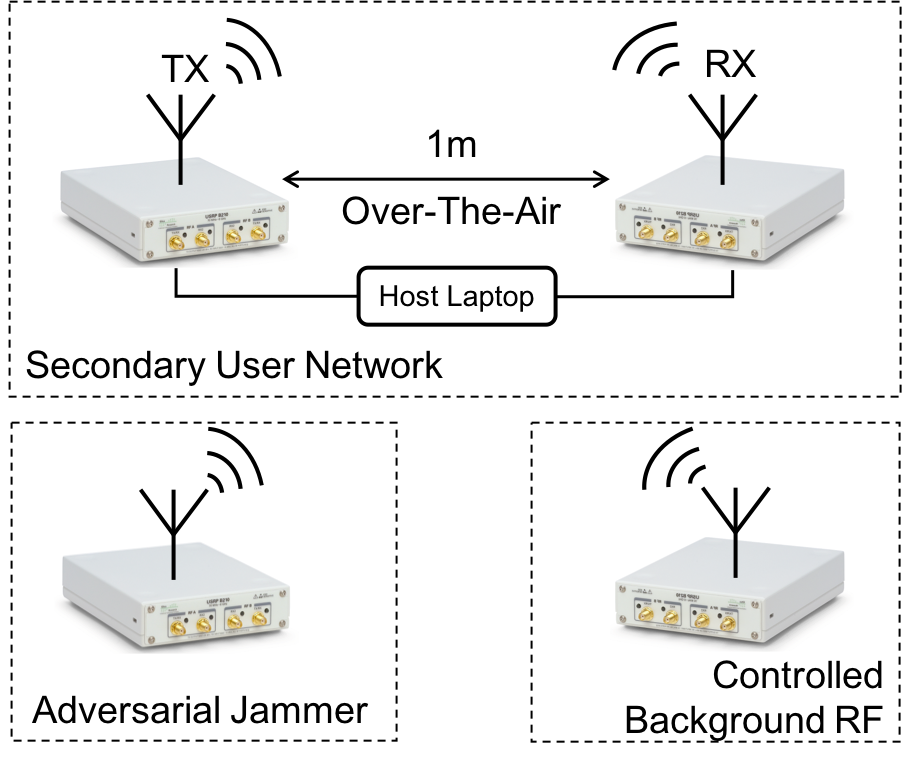}
	\caption{Cognitive radio testbed hardware setup. The receive node (RX) selects a channel according to its assigned policy and communicates the selected channel to the transmit node (TX) via an ACK channel. TX transmits a 500 kHz bandwidth QPSK signal in the 1.2 GHz UHF band. RX receives a reward based on the selected channel conditions.
	}
	\label{fig:usrp_setup}
	\end{figure}
\begin{figure}
	\centering
	\includegraphics[width=0.7\linewidth]{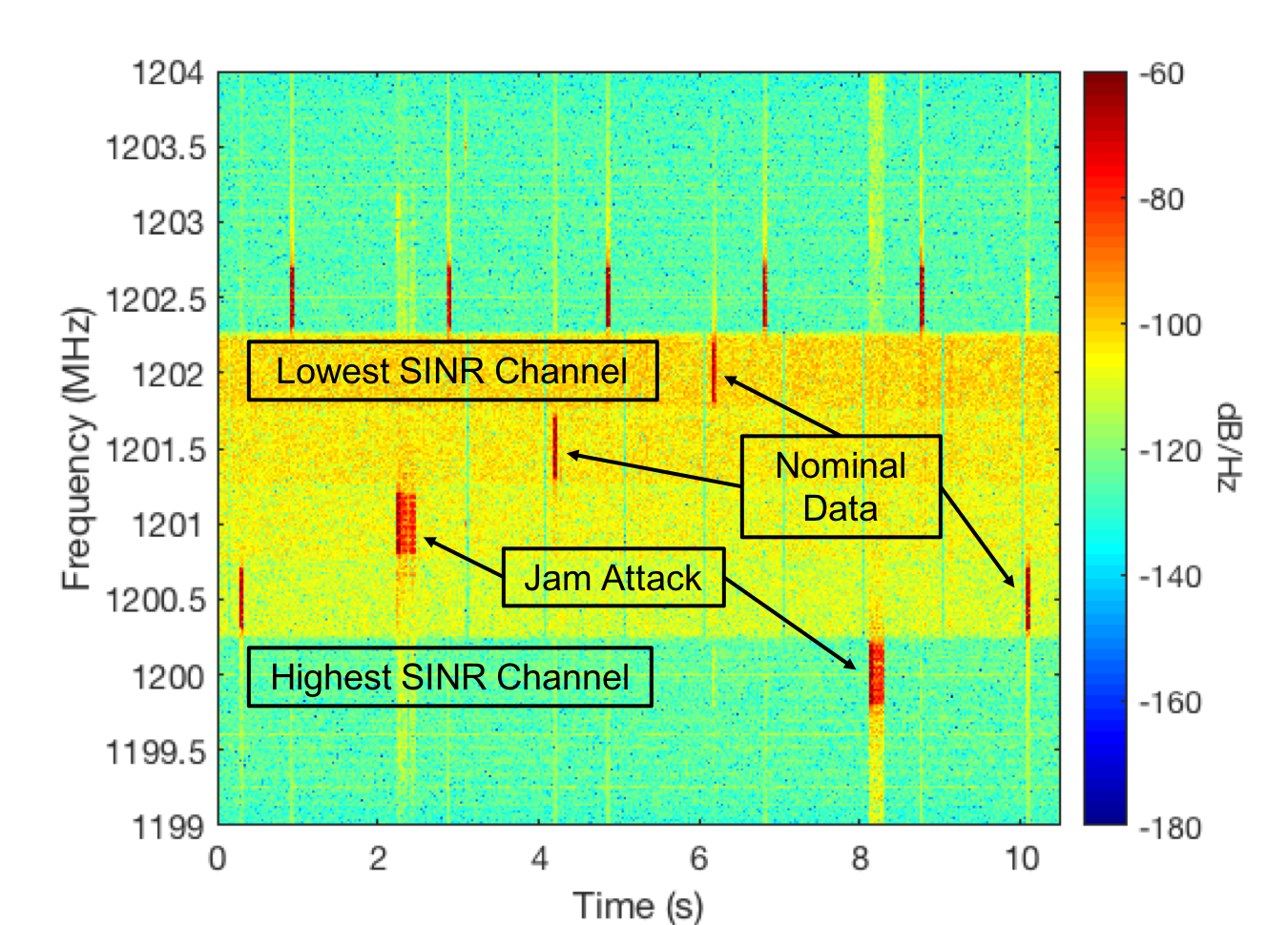} 
	\caption{Example spectrogram of received data.  Five individual 500 kHz channels are available for selection between center frequencies 1200 MHz and 1202 MHz with increasing amounts of background noise. Data transmission are sent every 2-seconds indicated by thin red data packets. Some proportion of data transmissions are overlaid with thicker red bars, indicating an adversarial attack.
	}\label{fig:spectrogram}\label{fig:test_setup}
\end{figure}

\begin{figure}[!htb]
	\centering
	\subfigure[$\rho= 0.125$]{
		\includegraphics[width=0.24\textwidth]{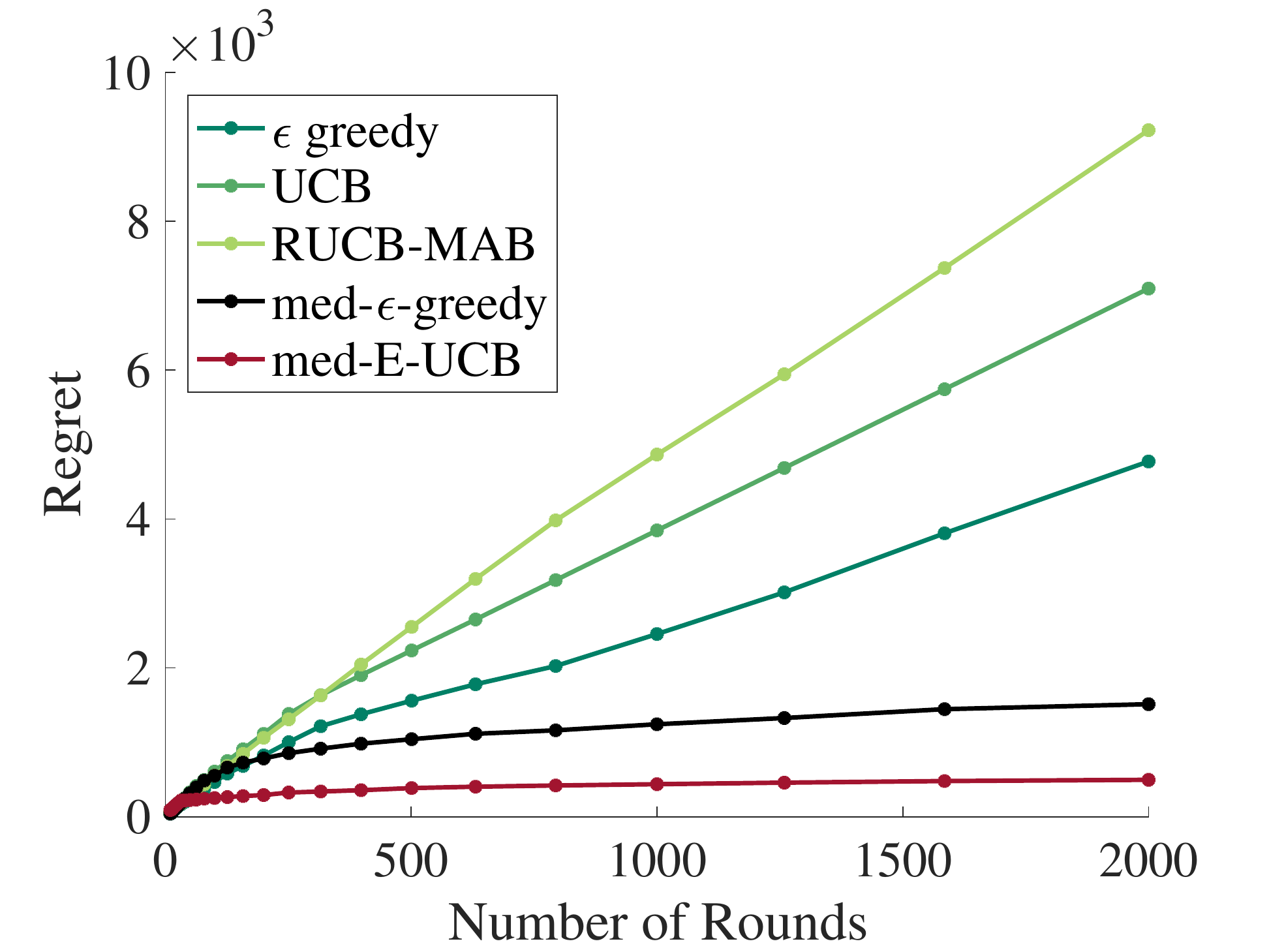}}\subfigure[$\rho= 0.3$]{
		\includegraphics[width=0.24\textwidth]{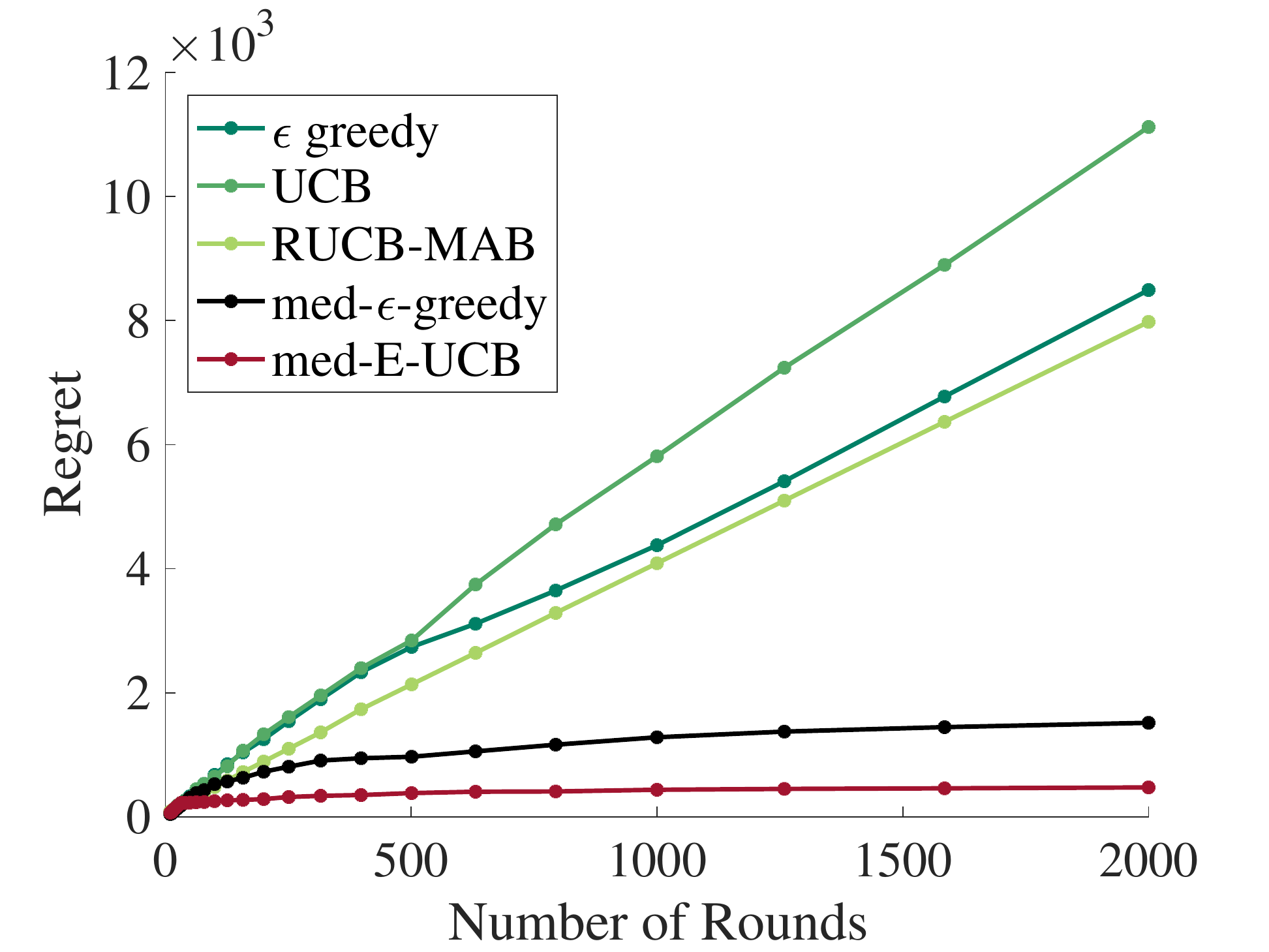}}
	\caption{Comparison of regret among algorithms}\label{fig:usr_algs}
\end{figure}
\begin{figure}[!htb]
	\centering
	\subfigure[$\rho= 0.125$]{
		\includegraphics[width=0.24\textwidth]{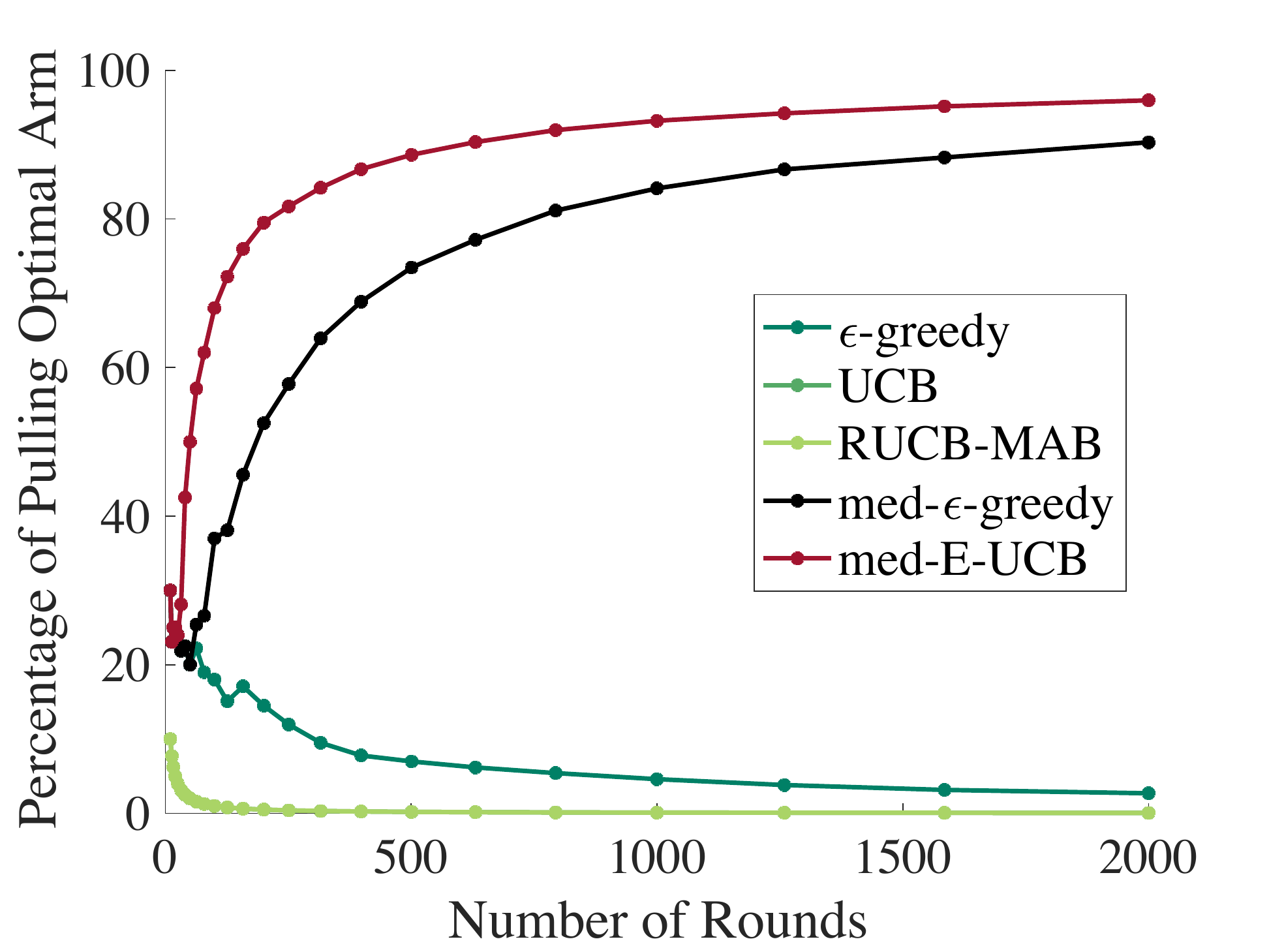}}\subfigure[$\rho= 0.3$]{
		\includegraphics[width=0.24\textwidth]{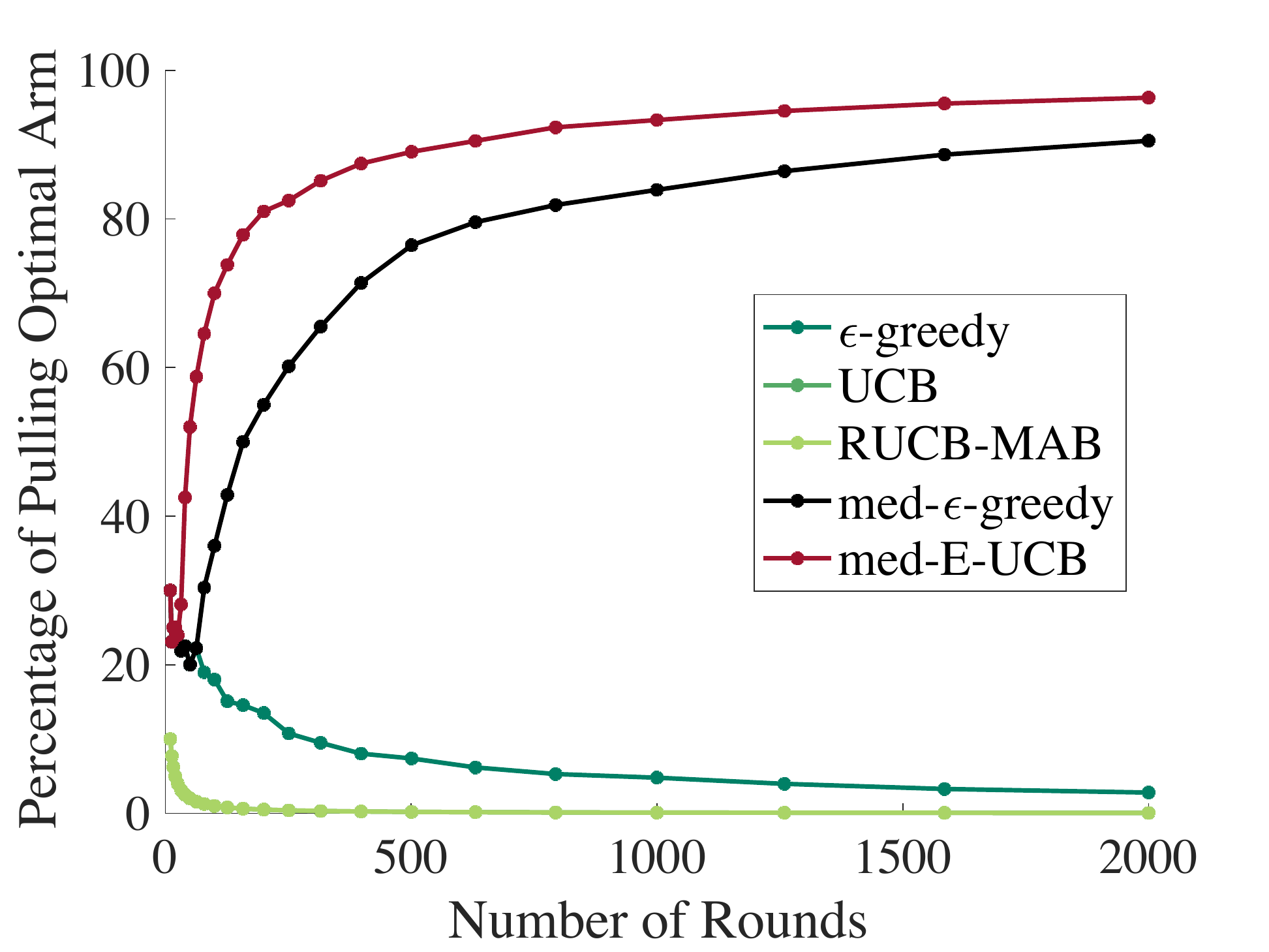}}
	\caption{Comparison of percentage of pulling the optimal arm among algorithms}\label{fig:usp_algs}
\end{figure} 

We center our experiment in the 1.2 GHz UHF frequency band across a 5 MHz RF bandwidth. Channel SINR is controlled using an additional USRP to transmit frequency varying noise to create five contiguous 500 kHz channels centered between 1200 MHz and 1202 MHz (\Cref{fig:test_setup}). The rewards of the 5 channels without adversarial perturbation are normally distributed with mean SINR of [41, 37, 35, 31, 28] dB and unit variance. Mean SINR of adversarially attacked rounds range between 5 to 10 dB. We use identical parameters as our simulations from \Cref{sec:simulation}, with $c = 10$ for med-$\epsilon$-greedy and $b = 4$, $\omega = 4$, for med-E-UCB. Similarly, the attack probabilities are fixed to be $\rho$ which equals either $0.125$ or $0.3$. Due to hardware timing constraints, rounds occur at 2-second intervals. So we reduce the total number of rounds $T = 2000$ for each experiment and set $G = 200$ for med-E-UCB. 

The experiment results are illustrated in \Cref{fig:usr_algs,fig:usp_algs}, which indicate that med-E-UCB and med-$\epsilon$-greedy achieve logarithmic regret compared with other algorithms including RUCB-MAB as well as mean UCB and mean $\epsilon$-greedy. Similarly, our algorithms both eventually converge to a 100\% pull rate of the optimal arm.

\section{Conclusion}

In this work, we proposed two median-based bandit algorithms, which we show to be robust under probabilistic unbounded valued adversarial attacks. Our median-based method can be potentially applied to many other models including multi-player bandits \cite{gai2011decentralized} and UCT \cite{kocsis2006bandit}.

\section*{Acknowledgment}

The work of Z. Guan, K. Ji and Y. Liang is supported in part by U.S.\ National Science Foundation under the grants CCF-1801846 and ECCS-1818904.

\bibliographystyle{aaai}\bibliography{references}

\onecolumn
\noindent\appendix{\LARGE \bf Supplementary Materials}
\section{Specification of Parameters for Algorithms in Experiments}\label{app:exp}
\begin{itemize}
	
	\item \textbf{med-E-UCB}: $b= 4, \omega=4$, and $G= 10000$
	
	\item \textbf{med-$\epsilon$-greedy}: $c = 10$
	
	\item \textbf{UCB}: refers to the ($\alpha, \psi$)-UCB algorithm proposed in \cite[Chap. 2.2]{bubeck2012regret}, with its penalty function $(\psi^*)^{-1}$ as recommended in \cite{bubeck2012regret}, and its parameter $\alpha$ tuned to the best in our experiment.
	
	\item \textbf{$\epsilon$-greedy}: refers to the algorithm proposed in \cite{auer2002finite}, with its exploration parameter $c$ tuned to the best in our experiment.
	
	\item \textbf{RUCB-MAB}: refers to the median-based vanilla UCB algorithm proposed in \cite{kapoor2019corruption}. Such an algorithm directly applies the sample median to replace the sample mean in UCB. 
	
	\item \textbf{EXP3}: refers to \cite{auer2002nonstochastic}. Given the time horizon $T$, we set the $\gamma =\min\left\{1, \sqrt{\frac{K\ln K}{(e-1)T}}\right\}$ as in corollary3.2 \cite{auer2002nonstochastic} suggests, and we take $g$ in the corollary  to be $T$.  
	
	\item \textbf{Cantoni UCB}: refers to the modified robust UCB algorithm with Cantoni's M-estimator proposed in \cite{bubeck2013bandits}, which can tolerate certain heavy tail outlier samples. 
	
	\item \textbf{$\alpha$-trimmed mean UCB}: refers to an $\alpha$-trimmed mean-based UCB algorithm, where we remove the top and bottom  $\alpha$ fraction of samples before calculating the sample mean. In our experiment, we choose $\alpha$ from $\{0.125, 0.3\}$ depending on the attack probability $\rho$, and tune other parameters to the best.
	
	\item \textbf{$\alpha$-trimmed mean $\epsilon$-greedy}: refers to an $\alpha$-trimmed mean-based $\epsilon$-greedy algorithm. We choose $\alpha$ from $\{0.125, 0.3\}$ depending on the attack probability $\rho$, and tune other parameters to the best in the experiment.

\end{itemize}

\section{Technical Lemmas}\label{sec:appen}
To prove the main results, we first establish several useful lemmas.  The following lemma shows that for each arm $i$, the proportion $s_i(m)$ of its attacked samples to its total $t$ samples  is  upper-bounded with high probability when $m$ is large enough. 
\begin{lemma}
	\label{lemma3}\label{lemma:2}
	Let $s_j(m) \defeq  \frac{\sum_{k\in Q_j(t), T_j(t) = m} \indicator{\eta_{k}\neq 0}}{m}$ 	denote the proportion of attacked data out of $m$ collected data from the $j$th arm.  Define an event $\mathcal{E} = \{ s_j(m_j) \le \rho + \epsilon_0 : m_j \ge N, j = 1,... K \}$, where $N =\ceiling{\frac{1}{2\epsilon_0^2} \log \frac{K}{ 2\epsilon_0^2\delta}}+1 \in \mathbb{N}$, with a small constant $0<\delta<1$. Then, we have $\prob{\mathcal{E}} \ge 1- \delta$. 
\end{lemma}
\begin{proof}[Proof of Lemma \ref{lemma3}]
	\begin{align*}
	\prob{\mathcal{E}} &= \prob{\bigcap\limits _{j=1}^K \bigcap\limits _{m_j = {N}} ^\infty \left\{ s_j(m_j) - \rho \le  \epsilon_0\right\}}= 1 - \prob{\bigcup\limits _{j=1}^K \bigcup\limits _{m _j= {N}} ^\infty \left\{ s_j(m_j) - \rho >  \epsilon_0\right\} } \nonumber\\
	&\overset{(i)}\ge 1 - \sum_{j=1}^K  \sum_{m _j= {N}}^\infty \prob{ s_j(m_j) - \rho > \epsilon_0}  \overset{(ii)}\ge 1 - \sum_{j=1}^K \sum_{m_j={N}}^{\infty} e^{-2\epsilon_0^2m_j}\\
	& \overset{(iii)}\ge 1- \sum_{j=1}^K \frac{1}{2\epsilon_0^2} \exp\bigparentheses{-2\epsilon_0^2(N-1)} \ge 1 - \sum_{j=1}^K \frac{\delta}{K}  = 1- \delta,
	\end{align*} 
	where ($i$) follows from the union bound, ($ii$) follows from Hoeffding's inequality \cite[See Theorem 2.2.2]{vershynin2018high},  and ($iii$) follows from the fact $\sum\limits_{t = x+1}^\infty e^{-Kt} \le \frac{1}{K} e^{-Kx}$.
\end{proof}
The following lemma provides a concentration bound on the $p$-quantile of attacked data $\{X_i\}_{i=1}^n$. 
\begin{lemma}
	\label{lemma4}
	Let $X_i  = \tilde{X}_i + \eta_i, i=1,....,n$ be $n$ attacked data samples, where $\tilde{X}_i, i=1,...,n$ are original (i.e., un-attacked)  data samples  i.i.d. drawn from the distribution with CDF $F$ and $\eta$ are unbounded attack values. Then, if  the ratio of attacked samples to total samples is upper-bounded by  $s$ (i.e., $\frac{1}{n}\sum_{i=1}^n \indicator{\eta_i \neq 0} \le s$) for constant $s\in (0,1)$. Then for any constants $p \in (s, 1-s)$, and $a,b > 0$, we  have
	\begin{align}
	&\prob{\pquantile{p}{\samplesetcom{X_i}{i=1}{n}}- \pquantile{p-s}{F} \le -a} \le \exp\bigparentheses{-2n[p-s - F(\pquantile{p-s}{F}-a)]^2}\nonumber \\
	&\prob{\pquantile{p}{\samplesetcom{X_i}{i=1}{n}}- \pquantile{p+s}{F} \ge b}\le \exp\bigparentheses{-2n[F(\pquantile{p+s}{F}+b)- p -s]^2}. \label{formula18}
	\end{align}
\end{lemma}
\begin{proof}[Proof of Lemma \ref{lemma4}]
	
	First, we claim 
	\begin{equation}
	\pquantile{p-s}{\samplesetcom{\tilde{X}_i}{i=1}{n}} \le  \pquantile{p}{\samplesetcom{X_i}{i=1}{n}} \le  \pquantile{p+s}{\samplesetcom{\tilde{X}_i}{i=1}{n}}, \label{formula16}
	\end{equation}
	which is quite obvious. As you can check the two extreme cases, where the adversary only add  $+\infty$ or $-\infty$ to the samples, the equality holds. Rigorous proof can be found in \cite[Lemma 3]{zhang2016median}.
	
	Based on the definition of quantile, we have 
	\begin{align}
	\prob{\pquantile{p-s}{\samplesetcom{\tilde{X}_i}{i=1}{n}} \le \pquantile{p-s}{F} -a}&\overset{(i)}=\prob{\sum_{i=1}^{n}\indicator{\tilde{X_i} \le \pquantile{p-s}{F} -a} \ge n(p-s)} \nonumber\\
	&\overset{(ii)}\le \exp\bigparentheses{-2n[p-s - F(\pquantile{p-s}{F}-a)]^2},  \label{formula14}
	\end{align}	
	where ($i$) follows from the definition of the $p-s$ quantile of the clean rewards, and $(ii)$ follows from \cite[Theorem 2.2.2]{vershynin2018high}. 
	Taking the steps similar to the above ones, we obtain 
	\begin{align}
	&\prob{\pquantile{p+s}{\samplesetcom{\tilde{X}_i}{i=1}{n}} \ge \pquantile{p+s}{F} +b}\le \exp\bigparentheses{-2n[F(\pquantile{p+s}{F}+b)-(p+s)]^2}. \label{formula15}
	\end{align}
	Combining \eqref{formula16}, \eqref{formula14} and \eqref{formula15}, we finish the proof.	
\end{proof}

\section{Proof of Theorem \ref{th:avg_mucb}}

We first define some notations. Let $Q_i(t) \defeq \{ \tau: \tau \le t, I_\tau = i\}$ be the set of rounds which consist of the time indices up to round $t$ in which the player pulls the $i$th arm. Let $T_i(t) \defeq |Q_i(t)|$ denote the number of rounds that the player pulls the $i$th arm up to round $t$, and let $\mathrm{med}_i(t) = \mathrm{med}\left(\{X_{i,\tau}\}_{\tau\in Q_i(t)}\right)$ denote the sample median of the collected data generated by the $i$th arm up to round $t$.

For any arm $j \neq i^*$ and any positive integer $l_0$, we have 
\begin{align}
T_j(T) &= \sum_{t=1}^T \indicator{I_t =j} = \sum_{ t\leq T, t \text{ is a pure exploration round}} \indicator{I_t =j } + \sum_{  t\leq T, t \text{ is a UCB round}} \indicator{I_t =j}\nonumber\\
&\overset{(i)}\le 1 + b\log (T+G) + \sum_{ t\leq T, t \text{ is a UCB round}} \indicator{I_t =j}  \nonumber \\
&\le 1+ b\log (T+G) + l_0  + \sum_{ t\leq T,  t \text{ is a UCB round}} \indicator{I_t =j, T_j(t-1) \ge l_0} \label{form:Tupper}, 
\end{align}
where ($i$) follows from the fact that  for each arm $j$,  the number of  pure exploration rounds is $\ceiling{b\log(\ceiling{\frac{T}{G}} \cdot G)} \le 1+ b\log(T + G)$. 
For notional simplicity, denote the UCB penalty term as $c_{t, T_i(t-1)} = \sqrt{\frac{\omega\log t}{T_i(t-1)}}$, where $\omega$ is defined in \Cref{alg:meucb}.  

Let $ \{ X_k^j \}_{k=1}^{m} := \{X_{j,k}\}_{k\in Q_j(t), T_j(t) :=m}$ be the first $m$ rewards collected from arm $j$. Then, based on~\eqref{form:Tupper}, we obtain
\begin{align}
T_j(T) &\overset{(i)}\le 1+ b\log(T+G) + l_0  \nonumber \\ 
& \;\; +  \sum_{t\leq T, t \text{ is a UCB round}} \indicator{\mathrm{med}_{i^*}(t-1) + c_{t-1, T_{i^*}(t-1)} \le \mathrm{med}_j(t-1) +c_{t-1, T_j(t-1)}, T_j(t-1) \ge l_0}  \nonumber  \\
& \le 1+ b\log(T+G) + l_0  + \sum_{t=l_0}^T \indicator{\mathrm{med}_{i^*}(t-1) + c_{t-1, T_{i^*}(t-1)} \le \mathrm{med}_j(t-1) +c_{t-1, T_j(t-1)}, T_j(t-1) \ge l_0} \nonumber  \\
& \overset{(ii)}\le 1+ b\log(T+G) + l_0 \nonumber \\
& \qquad + \sum_{t=l_0}^T \sum_{r= b\log t}^{t-1} \sum_{v =\max\{l_0, b\log t\}}^{t-1}\nonumber\\
&\qquad \mathbbm{1}\left\{\mathrm{med}_{i^*}(t-1) + c_{t-1, T_{i^*}(t-1)} \le \mathrm{med}_j(t-1) +c_{t-1, T_j(t-1)},  T_{i^*} (t-1)=r, T_j(t-1)= v \right\}\nonumber  \\
& = 1+ b\log (T+G) + l_0 + \sum_{t=l_0}^{T} \sum_{r= b\log t}^{t-1} \sum_{v =\max\{l_0, b\log t\}}^{t-1} \indicator{\pquantile{\frac{1}{2}}{\{X_k^{i^*}\}_{k=1}^r} + c_{t,r} \le \pquantile{\frac{1}{2}}{\{X_k^j\}_{k=1}^{v}}+c_{t, v}}, \label{form:Tupper2}
\end{align}

where ($i$) follows from the fact that  if arm $j$ is pulled at round $t$, then $\mathrm{med}_j(t -1)+c_{t-1,  T_j(t-1)}$ is the largest among all arms, and ($ii$) follows from the law of total probability.

Taking expectation with respect to the randomness of the regret on both sides of (\ref{form:Tupper2}), we obtain 
\begin{align}
\mathbb{E} (T_i(T)) &\le 1 + b\log(T+G) + l_0 \nonumber \\
&+ \sum_{t=l_0}^{T} \sum_{r= b\log t}^{t-1}\sum_{v =\max\{l_0, b\log t\}}^{t-1} \underbrace{\prob{\pquantile{\frac{1}{2}}{\{X_k^{i^*}\}_{k=1}^r} + c_{t,r} \le \pquantile{\frac{1}{2}}{\{X_k^j\}_{k=1}^{v}}+c_{t, v}}}_{(a)}. \label{form:TupperP}
\end{align}	
Let $\epsilon_0 = s-\rho$ and $l_0 = \ceiling{\frac{4\omega\log T}{(\pquantile{\frac{1}{2} -s}{F_{i^*}} - \pquantile{\frac{1}{2} + s}{F_j})^2}}.$  
Therefore, all $v$ in \eqref{form:TupperP} satisfy $v\geq  l_0\geq \ceiling{\frac{4\omega\log T}{(\pquantile{\frac{1}{2} -s}{F_{i^*}} - \pquantile{\frac{1}{2} + s}{F_j})^2}}$.

Define events 
\begin{align}\label{form:eventsplit}
\mathcal{A} &=\left\{\pquantile{\frac{1}{2}}{\{X_k^{i^*}\}_{k=1}^r} + c_{t,r} \le \pquantile{\frac{1}{2}}{\{X_k^j\}_{k=1}^{v}}+c_{t, v} \right\}, \nonumber\\
\mathcal{B} &=\left\{\pquantile{\frac{1}{2}}{\{X_k^{i^*}\}_{k=1}^r} + c_{t,r} \le \pquantile{\frac{1}{2} -s}{F_{i^*}} \right\},\nonumber \\ 
\mathcal{C} &= \left\{\pquantile{\frac{1}{2}}{\{X_k^j\}_{k=1}^{v}}\ge  \pquantile{\frac{1}{2} + s}{F_j} + c_{t, v}\right\}.
\end{align}
Next, we  show that $\mathcal{A}\subseteq\mathcal{B}\cup\mathcal{C}$ for all  $v \ge \ceiling{\frac{4\omega\log T}{(\pquantile{\frac{1}{2} -s}{F_{i^*}} - \pquantile{\frac{1}{2} + s}{F_j})^2}}$.

Suppose both $\mathcal{B}^c$ and $\mathcal{C}^c$ occur. Then, we have, 
\begin{align*}
\pquantile{\frac{1}{2}}{\{X_k^{i^*}\}_{k=1}^r} + c_{t,r}&\overset{(i)}> \pquantile{\frac{1}{2} -s}{F_{i^*}} = \pquantile{\frac{1}{2} + s}{F_j}  + \pquantile{\frac{1}{2} -s}{F_{i^*}} - \pquantile{\frac{1}{2} + s}{F_j} \\
& \overset{(ii)}\ge\pquantile{\frac{1}{2} + s}{F_j} + 2c_{t, v}\overset{(iii)}> \pquantile{\frac{1}{2}}{\{X_k^j\}_{k=1}^{v}} +c_{t, v},
\end{align*}
where ($i$) follows from the definition of $\mathcal{B}^c$,  ($ii$) follows from $c_{t, v} \le c_{t,  l_0}\le \frac{\pquantile{\frac{1}{2} -s}{F_{i^*}} - \pquantile{\frac{1}{2} + s}{F_j}}{2}$, and 
($iii$) follows from the definition of $\mathcal{C}^c$. The above inequality implies that 
$\mathcal{A}^c$ occurs, and thus $\mathcal{B}^c\cap\mathcal{C}^c \subseteq \mathcal{A}^c$,  and thus $\mathcal{A} \subseteq\mathcal{B}\cup \mathcal{C} $. Then, we have $\prob{\mathcal{A}}\leq \prob{\mathcal{B}}+\prob{\mathcal{C}}$.
We next upper-bound (a)  in~\eqref{form:TupperP}, i.e., $\prob{\mathcal{A}}$, by upper-bounding $\prob{\mathcal{B}}$ and $\prob{\mathcal{C}}$.

{\bf Upper-bounding $\prob{\mathcal{B}}$}: First note that 
\begin{align}\label{shav1}
&\prob{\mathcal{B}}=\prob{\pquantile{\frac{1}{2}}{\{X_k^{i^*}\}_{k=1}^r} + c_{t,r} \le \pquantile{\frac{1}{2} -s}{F_{i^*}}} \nonumber\\
&\qquad\overset{(i)}= \prob{\pquantile{\frac{1}{2}}{\{X_k^{i^*}\}_{k=1}^r} + c_{t,r} \le \pquantile{\frac{1}{2} -s}{F_{i^*}}, \quad s_{i^*}(r) \le \rho + \epsilon_0} \nonumber \\
& \qquad\qquad+ \prob{\pquantile{\frac{1}{2}}{\{X_k^{i^*}\}_{k=1}^r} + c_{t,r} \le \pquantile{\frac{1}{2} -s}{F_{i^*}}, \quad s_{i^*}(r) > \rho + \epsilon_0}  \nonumber\\
&\qquad = \prob{\pquantile{\frac{1}{2}}{\{X_k^{i^*}\}_{k=1}^r} + c_{t,r} \le \pquantile{\frac{1}{2} -s}{F_{i^*}} \middle| s_{i^*}(r) \le \rho + \epsilon_0}\prob{s_{i^*}(r) \le \rho + \epsilon_0} \nonumber\\
& \qquad\qquad+ \prob{\pquantile{\frac{1}{2}}{\{X_k^{i^*}\}_{k=1}^r} + c_{t,r} \le \pquantile{\frac{1}{2} -s}{F_{i^*}}, \quad s_{i^*}(r) > \rho + \epsilon_0}\nonumber\\
& \qquad  \overset{(ii)}\le \prob{\pquantile{\frac{1}{2}}{\{X_k^{i^*}\}_{k=1}^r} + c_{t,r} \le \pquantile{\frac{1}{2} -s}{F_{i^*}} \middle| s_{i^*}(r) \le \rho + \epsilon_0}  + \prob{s_{i^*}(r) > \rho + \epsilon_0},
\end{align}
where ($i$) follows from the law of total probability, and ($ii$) follows from the fact that $\prob{s_{i^*}(r) \le \rho + \epsilon_0} \le 1$ and $\mathbb{P}(\mathcal{S}\cap \mathcal{T})\leq \prob{\mathcal{T}}$ for any events $\mathcal{S}$ and $\mathcal{T}$. 

Applying Lemma \ref{lemma4} yields
\begin{align}
&\prob{\pquantile{\frac{1}{2}}{\{X_k^{i^*}\}_{k=1}^r} \le \pquantile{\frac{1}{2} -s}{F_{i^*}} - c_{t, r}} \le \exp\bigparentheses{-2r[\frac{1}{2}-s - F_{i^*}(\pquantile{\frac{1}{2} -s}{F_{i^*}} - c_{t,r})]^2}. \label{form:con}
\end{align}
Note that $c_{t,r} \le c_{t, b\log t}= \sqrt{\frac{\omega}{b}}\le \xi$, and    
\begin{align}
\frac{1}{2}-s - F_{i^*}(\pquantile{\frac{1}{2} -s}{F_{i^*}} - c_{t,r})&\ge \int_{\pquantile{\frac{1}{2} -s}{F_{i^*}}}^{\pquantile{\frac{1}{2} -s}{F_{i^*}} - c_{t,r}} \inf\{F'_{i^*}(x) : \ \pquantile{\frac{1}{2}-s}{F_{i^*}}-c_{t,r}< x < \pquantile{\frac{1}{2}-s}{F_{i^*}} \} dx \nonumber\\
& = \inf\{F'_{i^*}(x) : \ \pquantile{\frac{1}{2}-s}{F_{i^*}}-c_{t,r}< x < \pquantile{\frac{1}{2}-s}{F_{i^*}} \} c_{t,r}\overset{(i)}\ge lc_{t,r}, \label{form:trivalbound}
\end{align} 
where ($i$) follows the definition of $l$ in Assumption \ref{ass:mucb}. Substituting \eqref{form:trivalbound} into \eqref{form:con}, we obtain,  if event $\mathcal{E}$ occurs, 
\begin{align}\label{form:upb1}
\prob{\mathcal{B}} &= \prob{\pquantile{\frac{1}{2}}{\{X_k^{i^*}\}_{k=1}^r} +  c_{t,r}\le \pquantile{\frac{1}{2} -s}{F_{i^*}}}\le \bigparentheses{\frac{1}{t}}^{2\omega l^2}  \overset{(i)}\le \frac{1}{t^4},
\end{align}
where ($i$) follows because $\omega\ge \frac{2}{l^2}$.

In the meanwhile, applying \cite[Theorem 2.2.2]{vershynin2018high} to $s_{i^*}(s)$ yields 
\begin{align}\label{opopsa1}
\prob{s_{i^*}(r) > \rho + \epsilon_0} &\le \exp\bigparentheses{-2r\epsilon_0^2} \le \exp\bigparentheses{-2(b\log t)\epsilon_0^2} \overset{(i)}\le  \frac{1}{t^4},
\end{align}
where ($i$) follows from $b\ge \max\{\frac{2}{\epsilon_0^2}, \frac{\omega}{\xi^2}\}\ge \frac{2}{\epsilon_0^2}$. Combining~\eqref{form:upb1},~\eqref{opopsa1} and~\eqref{shav1} implies that 
\begin{align*}
\prob{\mathcal{B}}\le \frac{2}{t^4}.
\end{align*}

{\bf Upper-bounding $\prob{\mathcal{C}}$:} Taking  similar steps as in upper-bounding $\prob{\mathcal{B}}$, we have
\begin{align}
\prob{\mathcal{C}}=\prob{\pquantile{\frac{1}{2}}{\{X_k^j\}_{k=1}^{v}}\ge  \pquantile{\frac{1}{2} + s}{F_j} + c_{t, v}} \le \frac{2}{t^4} \nonumber
\end{align}

Thus, for all $r$ and $v$ in \eqref{form:TupperP}, we have 
\begin{align}\label{ggwrs1}
\prob{\pquantile{\frac{1}{2}}{\{X_k^{i^*}\}_{k=1}^r} + c_{t,r} \le \pquantile{\frac{1}{2}}{\{X_k^j\}_{k=1}^{v}}+c_{t, v}} \le \frac{4}{t^4}.
\end{align}
Based on~\eqref{ggwrs1}, we now upper-bound the pseudo-regret $\bar{R}_T$.

{\bf Upper-bounding $\bar{R}_T$:} Combining~\eqref{form:TupperP} and~\eqref{ggwrs1} yields 
\begin{align}\label{wopps1}
\mathbb{E} (T_j(T)) &\le  1 + b\log(T+G) + l_0\nonumber \\
&\quad+ \sum_{t=l_0}^{T} \sum_{r= b\log t}^{t-1} \sum_{v =\max\{l_0, b\log t\}}^{t-1} \prob{\pquantile{\frac{1}{2}}{\{X_k^{i^*}\}_{k=1}^r} + c_{t,r} \le \pquantile{\frac{1}{2}}{\{X_k^j\}_{k=1}^{v}}+c_{t, v}} \nonumber\\
& \le  1 + b\log(T+G) + l_0\sum_{t=l_0}^{T} \sum_{r= b\log t}^{t-1} \sum_{v =\max\{l_0, b\log t\}}^{t-1} \frac{4}{t^4}\nonumber \\
&\le 1 + b\log(T+G) + l_0 + \sum_{t=1}^{\infty} \sum_{r= 1}^{t} \sum_{v =1}^{t} \frac{4}{t^4} \nonumber\\
& \overset{(i)}\le 1 + b\log(T+G)+ l_0 + \frac{2\pi^2}{3} \nonumber \\
& \overset{(ii)}= 2 + b\log(T+G)+{\frac{4\omega\log T}{(\pquantile{\frac{1}{2} -s}{F_{i^*}} - \pquantile{\frac{1}{2} + s}{F_j})^2}} + \frac{2\pi^2}{3},
\end{align}
where ($i$) follows because $\sum_{t =1}^{\infty} \frac{1}{t^2} = \frac{\pi^2}{6}$ and ($ii$) follows from the fact that  $l_0 =\ceiling{\frac{4\omega\log T}{(\pquantile{\frac{1}{2} -s}{F_{i^*}} - \pquantile{\frac{1}{2} + s}{F_j})^2}}$.

Then, using~\eqref{wopps1}, we have 
\begin{align*}
\bar{R}_T &= \mathbb{E} [\sum_{t =1}^{T} \mu_{i^*} - \mu_{I_t}] = \mathbb{E}[\sum_{\substack{j=1 \\ j\neq i^*}}^{K} \sum_{t =1}^{T} \Delta_j \indicator{I_t =j}] = \sum_{\substack{j=1 \\ j\neq i^*}}^{K} \Delta_j\expectation{T_j(T)}\nonumber\\
& \overset{(i)}\le\sum_{j=1, j\neq i^*}^{K}\Delta_jb\log(T+G)+ \sum_{j=1, j\neq i^*}^{K} \Delta_j\frac{4\omega\log T}{(\pquantile{\frac{1}{2} -s}{F_{i^*}} - \pquantile{\frac{1}{2} + s}{F_j})^2} + \sum_{j=1, j\neq i^*}^{K}\Delta_j(2 + \frac{2\pi^2}{3}),
\end{align*}
where ($i$) follows from~\eqref{wopps1}. By exploiting the simple bounds on the constants, we derive following bound.
\begin{align}\label{woppsaaa1}
\bar{R}_T &\le\sum_{j=1, j\neq i^*}^{K}\Delta_j\left(b\log(2)+ \frac{4\omega}{(\pquantile{\frac{1}{2} -s}{F_{i^*}} - \pquantile{\frac{1}{2} + s}{F_j})^2} \right) \log T + \sum_{j=1, j\neq i^*}^{K}\Delta_j(2 + \frac{2\pi^2}{3}),
\end{align}
which completes the proof. 

\section{Proof of Theorem \ref{th:highprob_meu}}

For any non-optimal arm $j$, follow the same steps in the proof of \Cref{th:avg_mucb}. We obtain

\begin{align}
T_j(T) &\le 1+ b\log (T+G) + l_0 \nonumber \\
& \;\;+ \sum_{t=l_0}^{T} \sum_{r= b\log t}^{t-1} \sum_{v =\max\{l_0, b\log t\}}^{t-1} \indicator{\pquantile{\frac{1}{2}}{\{X_k^{i^*}\}_{k=1}^r} + c_{t,r} \le \pquantile{\frac{1}{2}}{\{X_k^j\}_{k=1}^{v}}+c_{t, v}}, \label{form:TjUpperP}
\end{align}
for any positive $l_0$.

Let $\epsilon_0 = s-\rho$ and $l^j_0 = \max\Big\{\ceiling{\frac{4\omega\log T}{(\pquantile{\frac{1}{2} -s}{F_{i^*}} - \pquantile{\frac{1}{2} + s}{F_j})^2}}, \exp(\frac{2}{b})\bigparentheses{\frac{K}{\delta\epsilon_0^2}}^{\frac{1}{2\epsilon_0^2b}}, \frac{2K}{\delta} +1\Big\}.$ 
Then, all $v$ in \eqref{form:TjUpperP} satisfy $v\geq  l^j_0\geq \ceiling{\frac{4\omega\log T}{(\pquantile{\frac{1}{2} -s}{F_{i^*}} - \pquantile{\frac{1}{2} + s}{F_j})^2}}$.

Define events $\mathcal{A}$, $\mathcal{B}$ and $\mathcal{C}$ as we did in \eqref{form:eventsplit}. As we have shown in the proof of \Cref{th:avg_mucb}, $\mathcal{A}\subseteq\mathcal{B}\cup\mathcal{C}$ for all  $v \ge \ceiling{\frac{4\omega\log T}{(\pquantile{\frac{1}{2} -s}{F_{i^*}} - \pquantile{\frac{1}{2} + s}{F_j})^2}}$. This implies 

\begin{align}
T_j(T) &\le  1+ b\log (T+G) + l^j_0 \nonumber \\
& \;\;+ \sum_{t=l_0}^{T} \sum_{r= b\log t}^{t-1} \sum_{v =\max\{l_0, b\log t\}}^{t-1} \indicator{\pquantile{\frac{1}{2}}{\{X_k^{i^*}\}_{k=1}^r} + c_{t,r} \le \pquantile{\frac{1}{2} -s}{F_{i^*}}} + \indicator{\pquantile{\frac{1}{2}}{\{X_k^j\}_{k=1}^{v}}\ge  \pquantile{\frac{1}{2} + s}{F_j} + c_{t, v}} \nonumber \\
&  \le 1+ b\log (T+G) + l^j_0  + \sum_{t=l_0}^{T} t\sum_{r= b\log t}^{t-1} \indicator{\pquantile{\frac{1}{2}}{\{X_k^{i^*}\}_{k=1}^r} + c_{t,r} \le \pquantile{\frac{1}{2} -s}{F_{i^*}}} \nonumber \\
&  \qquad +\sum_{t=l_0}^{T} t \sum_{v = b\log t}^{t-1} \indicator{\pquantile{\frac{1}{2}}{\{X_k^j\}_{k=1}^{v}}\ge  \pquantile{\frac{1}{2} + s}{F_j} + c_{t, v}}. \nonumber 
\end{align}

Therefore, we derive the following upper bound of regret,

\begin{align}
\mathcal{R}_T  &= \sum_{t=1}^T \mu^* - \mu_{I_t} = \sum_{j=1, j\neq i^*}\Delta_j T_j(T) \nonumber \\
& \overset{(i)}= \sum_{j=1, j\neq i^*} \Delta_j (1+ b\log (T+G) + l^j_0)  + \sum_{j=1, j\neq i^*} \Delta_j \sum_{t=l_0}^{T} t \sum_{v =b\log t}^{t-1} \indicator{\pquantile{\frac{1}{2}}{\{X_k^j\}_{k=1}^{v}}\ge  \pquantile{\frac{1}{2} + s}{F_j} + c_{t, v}} \nonumber \\
& \qquad + \sum_{j=1, j\neq i^*} \Delta_j  \sum_{t=l_0}^{T} t\sum_{r= b\log t}^{t-1} \indicator{\pquantile{\frac{1}{2}}{\{X_k^{i^*}\}_{k=1}^r} + c_{t,r} \le \pquantile{\frac{1}{2} -s}{F_{i^*}}},
\end{align}
where $(i)$ follows the upper bound of $T_j(T)$.

Define event $\mathcal{G}$ as follows.

\begin{align*}
\mathcal{G} &=\bigparentheses{\bigcap_{j=1, j\neq i^*} \bigcap_{t=l_0}^T\bigcap_{v=b\log t}^{t-1} \left\{\indicator{\pquantile{\frac{1}{2}}{\{X_k^j\}_{k=1}^{v}}\ge  \pquantile{\frac{1}{2} + s}{F_j} + c_{t, v}}\le \frac{1}{t^4}\right\}} \\
&\qquad \bigcap \bigparentheses{\bigcap_{t=l_0}^T\bigcap_{v=b\log t}^{t-1} \left\{\indicator{\pquantile{\frac{1}{2}}{\{X_k^{i^*}\}_{k=1}^r} + c_{t,r} \le \pquantile{\frac{1}{2} -s}{F_{i^*}}}\le \frac{1}{t^4}\right\}}
\end{align*}

Under event $\mathcal{G}$, we obtain
\begin{align*}
\mathcal{R}_T &\le  \sum_{j=1, j\neq i^*}\Delta_j \bigparentheses{1+ b\log (T+G) + l^j_0 + \frac{\pi^2}{3}} \\	
&\le \sum_{j=1, j\neq i^*} \Delta_j \bigparentheses{3 + \frac{\pi^2}{3} + b\log(2)\log(T)  + \frac{4\omega\log(T)}{(\pquantile{\frac{1}{2} -s}{F_{i^*}} - \pquantile{\frac{1}{2} + s}{F_j})^2} + \exp(\frac{2}{b})(\frac{K}{\delta\epsilon^2})^\frac{1}{2\epsilon_0^2b}+ \frac{2K}{\delta} } \\
&\overset{(i)}\le \sum_{j=1, j\neq i^*}\Delta_j \bigparentheses{b\log(2) + \frac{4\omega\log(T)}{(\pquantile{\frac{1}{2} -s}{F_{i^*}} - \pquantile{\frac{1}{2} + s}{F_j})^2}}\log T \\
&\qquad + \sum_{j=1, j\neq i^*} \Delta_j\bigparentheses{e\bigparentheses{\frac{bK}{2\delta}}^\frac{1}{4} + \frac{2K}{\delta}} + \sum_{j=1, j\neq i^*}\Delta_j(3 + \frac{\pi^2}{3}),
\end{align*}
where $(i)$ follows from our assumptions of the parameters.

We nex lower-bound the probability of event $\mathcal{G}$.  Note that by Lemma 1 and setting $N = \ceiling{\frac{1}{2\epsilon_0^2} \log \frac{K}{\epsilon_0^2\delta}}+1$, event $\mathcal{E} = \{ s_j(m_j) \le \rho + \epsilon_0 : m_j \ge N, j = 1,... K \}$ happens with probability at least $1- \frac{\delta}{2}$.
Based on this fact, we have 
\begin{align}
\prob{\mathcal{G}^c} &= \prob{\mathcal{G}^c, \mathcal{E}} + \prob{\mathcal{G}^c, \mathcal{E}^c }  \le  \prob{\mathcal{G}^c|\mathcal{E}} +  \prob{\mathcal{E}^c }   = \prob{\mathcal{G}^c|\mathcal{E}}  + \frac{\delta}{2}.\label{form:probG}
\end{align}

And since $\prob{\mathcal{U}\cup \mathcal{T}} \le \prob{\mathcal{U}}+ \prob{\mathcal{T}}$ for any events $\mathcal{U}$ and $\mathcal{T}$, we obtain

\begin{align*}
\prob{\mathcal{G}^c|\mathcal{E}} &\le \sum_{j=1, j\neq i^*} \sum_{t=l_0}^T \sum_{v=b\log t}^{t-1} \prob{\indicator{\pquantile{\frac{1}{2}}{\{X_k^j\}_{k=1}^{v}}\ge  \pquantile{\frac{1}{2} + s}{F_j} + c_{t, v}} > \frac{1}{t^4} | \mathcal{E}}\\
& \qquad + \sum_{t=l_0}^T \sum_{r=b\log t}^{t-1} \prob{ \indicator{\pquantile{\frac{1}{2}}{\{X_k^{i^*}\}_{k=1}^r} + c_{t,r} \le \pquantile{\frac{1}{2} -s}{F_{i^*}}} > \frac{1}{t^4} | \mathcal{E}} \\
& \overset{(i)}\le  \sum_{j=1, j\neq i^*} \sum_{t=l_0}^T \sum_{v=b\log t}^{t-1}t^4 \expectation{\indicator{\pquantile{\frac{1}{2}}{\{X_k^j\}_{k=1}^{v}}\ge  \pquantile{\frac{1}{2} + s}{F_j} + c_{t, v}}|\mathcal{E}} \\
& \qquad + \sum_{t=l_0}^T \sum_{r=b\log t}^{t-1}t^4 \expectation{\indicator{\pquantile{\frac{1}{2}}{\{X_k^{i^*}\}_{k=1}^r} + c_{t,r} \le \pquantile{\frac{1}{2} -s}{F_{i^*}}} | \mathcal{E}} \\
& \overset{(ii)}\le \sum_{j=1, j\neq i^*} \sum_{t=l_0}^T  \bigparentheses{\frac{1}{t}}^{2\omega l^2-5}  +  \sum_{t=l_0}^T  \bigparentheses{\frac{1}{t}}^{2\omega l^2-5}  \overset{(iii)}\le K  \sum_{t=l_0}^T \frac{1}{t^2}   \le \frac{K}{l_0-1}  \overset{(iv)}\le \frac{\delta}{2},
\end{align*}
where $(i)$ follows from the Markov inequality, $(ii)$ follows from the steps similar to (15) and (16), in our paper, $(iii)$ follows from the fact that $\omega > \frac{3.5}{l^2}$, and $(iv)$ follows from the fact that $l_0 \ge \frac{2K}{\delta} +1$.

Combining with \ref{form:probG}, we have $\prob{G} \ge 1 - \delta$.

\section{Proof of \Cref{cor:mucb}}

To prove this corollary, we first show that the given Gaussian distributions meet Assumption \ref{ass:mucb}, and then apply \Cref{th:avg_mucb} and \Cref{th:highprob_meu}  to complete the proof.

Let $s =\Phi(\frac{\Delta_{min}}{4\sigma})-\frac{1}{2}$. Then we have 
$$\pquantile{\frac{1}{2}-s}{F_{i^*}} = \mu^* - \frac{\Delta_{min}}{4}.$$
And for all $j\neq i^*$, we have 
$$\pquantile{\frac{1}{2}+s}{F_j} = \mu_j + \frac{\Delta_{min}}{4}.$$ 
Thus, 
$$\pquantile{\frac{1}{2}-s}{F_{i^*}} -\pquantile{\frac{1}{2}+s}{F_j} = \Delta_j - \frac{\Delta_{min}}{2} \ge \frac{\Delta_{min}}{2} >0. $$

The CDF of Gaussian distributions is differentiable at every point of $\mathbb{R}$, which meet the requirement of differentiability.

Let $\xi = 1$, and $l = \frac{1}{\sqrt{2\pi\sigma^2}}\exp\bigparentheses{-\frac{(\Delta_{min}+4)^2}{32\sigma^2} }$. It is easy to check that, for the optimal arm,
$$\left. \inf\{F'_{i^*}(x) : \ \pquantile{\frac{1}{2}-s}{F_{i^*}}-\xi< x < \pquantile{\frac{1}{2}-s}{F_{i^*}} \}= l\right.,$$
and for all $j\neq i^*$
$$\inf\{F'_{j}(x) : \ \pquantile{\frac{1}{2}+s}{F_j}< x <  \pquantile{\frac{1}{2}+s}{F_{j}}+\xi\} = l.$$ 	

Therefore, the given Gaussian distributions meet all requirements of Assumption \ref{ass:mucb}. If $\rho<\Phi(\frac{\Delta_{min}}{4\sigma})-\frac{1}{2}$, let $b \ge \{\omega, \frac{2}{(\Phi(\Delta_{min}/(4\sigma)) -1/2 -\rho)^2}\}$ and $\omega \ge \frac{2}{l^2}$. Applying \Cref{th:avg_mucb} and  \Cref{th:highprob_meu}, we complete the proof.

\section{Proof of Theorem \ref{th: avg_meg}}
Our first step is to upper-bound the probability $\mathbb{P}(I_t = j)$ for any $j\neq i^*$, and $t\ge (\ceiling{c}K +1)e:= A$. Based on Algorithm~\ref{alg: meg}, we have 
\begin{align}\label{iops}	
\mathbb{P}(I_t = j) &= \mathbb{P} (\text{explore in round $t$, draw arm j}) + \mathbb{P}(\text{exploit in round $t$, draw arm j}) \nonumber\\
&= \frac{c}{t} + (1- \frac{cK}{t}) \cdot \mathbb{P}(\mathrm{med}_j (t -1) = \max_j{\mathrm{med}_j (t-1)}) \nonumber\\
& \le \frac{c}{t} + (1- \frac{cK}{t}) \cdot \mathbb{P}(\mathrm{med}_j (t -1) \ge \mathrm{med}_{i^*} (t-1))\nonumber\\
& \le \frac{c}{t} +\mathbb{P}(\mathrm{med}_j (t -1) \ge \mathrm{med}_{i^*} (t-1)).
\end{align}	
Define events $\mathcal{A} =\{\mathrm{med}_j(t-1) \ge \mathrm{med}_{i^*}(t-1)\}$, $\mathcal{B} =\{ \mathrm{med}_j(t-1) \ge x_0\} $ and  $\mathcal{C} = \{\mathrm{med}_{i^*}(t-1) \le x_0\}$, where $x_0$ is defined in Assumption \ref{ass:meg}. Next, we show that $\mathcal{A}\subseteq\mathcal{B}\cup\mathcal{C}$. 

Assuming both $\mathcal{B}^c$ and $\mathcal{C}^c$ hold, we have 
\begin{equation*}
\mathrm{med}_j(t-1) < x_0 < \mathrm{med}_{i^*} (t-1),
\end{equation*}
which implies that $\mathcal{A}^c$ is true. Thus, we have  $\mathcal{B}^c\cap\mathcal{C}^c \subseteq \mathcal{A}^c$. Taking complementary on both sides implies $\mathcal{A}\subseteq \mathcal{B}\cup\mathcal{C}$. Thus, we have $\prob{\mathcal{A}}\le \prob{\mathcal{B}}+\prob{\mathcal{C}}$, which in conjunction with~\eqref{iops}, implies 
\begin{align}\label{fpos}
\mathbb{P}(I_t = j) \leq \frac{c}{t} + \mathbb{P}(\mathcal{B})+\mathbb{P}(\mathcal{C}). 
\end{align}
Then, our next two steps are to upper-bound $\prob{\mathcal{B}}$ and $\mathbb{P}(\mathcal{C})$, respectively.

{\bf Upper-bounding $\prob{\mathcal{B}}:$} Let $ \{ X_k^j \}_{k=1}^{m} := \{X_{j,k}\}_{k\in Q_j(t), T_j(t) :=m}$ be the first $m$ rewards collected from arm $j$, $T_j^R(t)= \sum_{\tau=1}^{t} \indicator{I_\tau = j,\ \text{exploring in round $\tau$}}$, and $u(t)  = \sum_{k = \ceiling{c}K+1}^{t-1} \frac{c}{2k}$. Then, we have 
\begin{align}
\mathbb{P} \left( \mathrm{med}_j(t-1) \ge x_0 \right) &= \sum_{\tau=1}^{t-1} \prob{\mathrm{med}_j(t-1) \ge x_0, T_j(t-1) = \tau}= \sum_{\tau=1}^{t-1} \prob { \medthta{\{X_k^j\}_{k=1}^\tau} \ge x_0, T_j(t-1) = \tau} \nonumber\\
&= \sum_{\tau=1}^{t-1} \prob {T_j(t-1) = \tau  \middle|  \medthta{\{X_k^j\}_{k=1}^\tau }\ge x_0}\prob{\medthta{\{X_k^j\}_{k=1}^\tau} \ge x_0}  \nonumber\\
& \overset{(i)}\le\sum_{\tau=1}^{t-1} \prob {T_j^R(t-1) \le \tau  \middle|  \medthta{\{X_k^j\}_{k=1}^\tau }\ge x_0}  \prob{\medthta{\{X_k^j\}_{k=1}^\tau} \ge x_0} \nonumber\\
&\overset{(ii)}=\sum_{\tau=1}^{t-1} \prob {T_j^R(t-1) \le \tau } \cdot \prob{\medthta{\{X_k^j\}_{k=1}^\tau} \ge x_0} \nonumber\\
& = \sum_{\tau=1}^{\flooring{u(t)}} \prob {T_j^R(t-1) \le \tau } \cdot \prob{\medthta{\{X_k^j\}_{k=1}^\tau} \ge x_0}\nonumber\\
&  \qquad+ \sum_{\tau=\flooring{u(t)}+1}^{t-1} \prob {T_j^R(t-1) \le \tau } \cdot \prob{\medthta{\{X_k^j\}_{k=1}^\tau} \ge x_0} \nonumber\\
& \le  \sum_{\tau=1}^{\flooring{u(t)}} \prob {T_j^R(t-1) \le \tau } + \sum_{\tau=\flooring{u(t)}+1}^{t-1} \prob{\medthta{\{X_k^j\}_{k=1}^\tau} \ge x_0} \nonumber\\
&  \overset{(iii)} \le \underbrace{  u(t)\prob {T_j^R(t-1) \le u(t)}}_{(a)}+ \underbrace{ \sum_{\tau=\flooring{u(t)}+1}^{t-1} \prob{\medthta{\{X_k^j\}_{k=1}^\tau} \ge x_0}}_{(b)}, \label{form: split1}
\end{align}
where ($i$) follows from  the fact $T_j^R(t-1) \le T_j(t-1)$, ($ii$) follows from the fact that $T_j^R(t-1) $ and $\medthta{\{ X_k^j\}_{k=1}^\tau} $  are mutually  independent, and ($iii$) follows from the fact that $\prob {T_j^R(t-1) \le \tau} \le \prob {T_j^R(t-1) \le u(t)}$ for all $\tau \le u(t)$.
{\bf Upper-bounding (a) in~\eqref{form: split1}}:
Since $T_j^R(t-1)$ can be rewritten as $\sum_{k=\ceiling{c}K+1}^{t-1} r_k$ with $r_k \sim \text{Bernoulli}(\frac{c}{k})$, we have $\expectation{T_j^R(t-1)} = \sum_{k=\ceiling{c}K+1}^{t-1}\frac{c}{k} = 2u(t)$ and $\mathrm{Var}(T_j^R(t-1)) = \sum_{k=\ceiling{c}K+1}^{t} 2\cdot\frac{c}{k}(1-\frac{c}{k}) < 2u(t)$. Then, applying the variant of Bernstein inequality~\cite[Fact 2]{auer2002finite} to $T_j^R(t-1)$ yields  	
\begin{align*}
\prob{T_j^R(t-1) \le u(t)} &= \prob{T_j^R(t-1)\le \expectation{T_j^R(t-1)} - u(t)} \le \exp\bigparentheses{-\frac{u(t)^2/2}{\mathrm{Var}(T_j^R(t-1))+ u(t)/2}} \le \exp\bigparentheses{-\frac{u(t)}{5}}
\end{align*}
which implies that (a) in~\eqref{form: split1} is upper-bounded by 
\begin{align}\label{form: medupa}
(a) \le u(t)\exp\bigparentheses{-\frac{u(t)}{5}}.
\end{align}
Since $t \ge A \ge (\ceiling{c}K +1)\exp(\frac{10}{c})$, we have $u(t) > \sum_{k = \ceiling{c}K+1}^{t-1} \frac{c}{2} \log\frac{k+1}{k} = \frac{c}{2}\log \frac{t}{\ceiling{c}K+1} \ge 5$, which combined with~\eqref{form: medupa} and the fact that function $f(x)= x\exp\bigparentheses{-x/5}$ is decreasing for any $x\ge5$,  yields 
\begin{align}\label{form: aupper}
(a) &\le \frac{c}{2}\bigparentheses{\log \frac{t}{\ceiling{c}K+1}}\bigparentheses{\frac{\ceiling{c}K+1}{t}}^{\frac{c}{10}}  \le \frac{c\bigparentheses{\ceiling{c}K+1}^{\frac{c}{10}}}{2}\log\bigparentheses{t}\bigparentheses{\frac{1}{t}}^{\frac{c}{10}}.
\end{align}

{\bf Upper-bounding (b) in~\eqref{form: split1}}:
Let $\epsilon_0 = s -\rho$. We first note that 
\begin{align}\label{ggwin}
\prob{\medthta{\{X_k^j\}_{k=1}^\tau} \ge x_0} &\overset{(i)}=\prob{\medthta{\{X_k^j\}_{k=1}^\tau} \ge x_0, s_j(\tau) \le \rho + \epsilon_0} + \prob{\medthta{\{X_k^j\}_{k=1}^\tau} \ge x_0, s_j(\tau) > \rho + \epsilon_0} \nonumber \\	
&  = \mathbb{P} \left(\medthta{\{X_k^j\}_{k=1}^\tau} \ge x_0 \middle|  s_j(\tau) \le \rho + \epsilon_0\right)\prob{s_j(\tau) \le \rho +\epsilon_0} \nonumber\\
&  \quad \quad+ \prob{\medthta{\{X_k^j\}_{k=1}^\tau} \ge x_0, s_j(\tau) > \rho + \epsilon_0}  \nonumber\\
&  \overset{(ii)}\le   \mathbb{P} \left(\medthta{\{X_k^j\}_{k=1}^\tau} \ge x_0 |  s_j(\tau) \le \rho + \epsilon_0\right)  + \prob{s_j(\tau) > \rho + \epsilon_0},
\end{align}
where ($i$) follows from the law of total probability, and ($ii$) follows from the fact that $\prob{s_j(\tau) \le \rho+\epsilon_0}\leq 1$ and  $\mathbb{P}(\mathcal{S}\cap \mathcal{T})\leq \prob{\mathcal{T}}$ for any events $\mathcal{S}$ and $\mathcal{T}$. Then, applying \Cref{lemma4} to the first term on the right side of~\eqref{ggwin} yields
\begin{align}\label{opsa}
\mathbb{P} &\left(\medthta{\{X_k^j\}_{k=1}^\tau } \ge x_0 | s_j(\tau) \le \rho + \epsilon_0 \right)\le\exp \bigparentheses{-2\tau[F_j(x_0) -\frac{1}{2} -s]^2}.
\end{align}
In the meanwhile, applying \cite[Theorem 2.2.2]{vershynin2018high} to the second term yields
\begin{align}\label{opsb}
\prob{s_j(\tau) > \rho + \epsilon_0} \le \exp\bigparentheses{-2\tau\epsilon_0^2}.
\end{align}
Thus, combining~\eqref{opsa},~\eqref{opsb} and~\eqref{ggwin}, we have 
\begin{align}\label{wori} 
(b) & \le \sum_{\tau = \flooring{u(t)} +1}^{t-1}\bigparentheses{\exp \bigparentheses{-2\tau[F_j(x_0) -\frac{1}{2} -s]^2} + \exp\bigparentheses{-2\tau\epsilon_0^2}} \nonumber\\
&  \overset{(i)}\le \frac{1}{2\bigparentheses{F_j(x_0)-\frac{1}{2} -s}^2}\exp\bigparentheses{-2(F_j(x_0) - \frac{1}{2} -s)^2(u(t)-1)} + \frac{1}{2\epsilon_0^2}\exp\bigparentheses{-2\epsilon_0^2(u(t)-1)} \nonumber\\
&   \overset{(ii)}\le \frac{\exp\bigparentheses{2\bigparentheses{F_j(x_0)-\frac{1}{2} -s}^2}}{2\bigparentheses{F_j(x_0)-\frac{1}{2} -s}^2}\left(\frac{\ceiling{c}K+1}{t}\right)^{\bigparentheses{F_j(x_0)-\frac{1}{2} -s}^2c} + \frac{\exp(2\epsilon_0^2)}{2\epsilon^2_0}\bigparentheses{\frac{\ceiling{c}K+1}{t}}^{\epsilon_0^2c},
\end{align}
where ($i$) follows because $\sum\limits_{t = x+1}^\infty e^{-Kt} \le \frac{1}{K} e^{-Kx}$, and ($ii$) follows from the fact $u(t) > \sum_{k = \ceiling{c}K+1}^{t-1} \frac{c}{2} \log\frac{k+1}{k} =\frac{c}{2} \log\frac{t}{\ceiling{c}K+1} $.

{\bf Upper-bounding  $\prob{\mathcal{C}}$}: 
Similar to steps as in \crefrange{form: split1}{wori}, we have, for any $t\geq A$, 
\begin{align}
\prob{\mathcal{C}}\le& \frac{c\bigparentheses{\ceiling{c}K+1}^{\frac{c}{10}}}{2}\log\bigparentheses{t}\bigparentheses{\frac{1}{t}}^{\frac{c}{10}}\nonumber\\
&\quad  +  \frac{\exp\bigparentheses{2\bigparentheses{\frac{1}{2} -s - F_{i^*}(x_0)}^2}}{2\bigparentheses{\frac{1}{2} -s - F_{i^*}(x_0)}^2}\left(\frac{\ceiling{c}K+1}{t}\right)^{\bigparentheses{\frac{1}{2} -s-F_{i^*}(x_0)}^2c} + \frac{\exp(2\epsilon_0^2)}{2\epsilon^2_0}\bigparentheses{\frac{\ceiling{c}K+1}{t}}^{\epsilon_0^2c}.\label{uppc}
\end{align}

Combining the above two upper bounds \eqref{wori} and \eqref{uppc} with $P(I_t =j)\le\frac{c}{t} +  \prob{\mathcal{B}}+ \prob{\mathcal{C}}$ yields , for any $t\geq A$ 
\begin{align}\label{form: medIt2}
\prob{I_t = j } &\le \frac{c}{t} +c\bigparentheses{\ceiling{c}K+1}^{\frac{c}{10}}\log\bigparentheses{t}\bigparentheses{\frac{1}{t}}^{\frac{c}{10}} + \frac{\exp\bigparentheses{2\bigparentheses{F_j(x_0)-\frac{1}{2} -s}^2}}{2\bigparentheses{F_j(x_0)-\frac{1}{2} -s}^2}\left(\frac{\ceiling{c}K+1}{t}\right)^{\bigparentheses{F_j(x_0)-\frac{1}{2} -s}^2c} \nonumber\\
&\qquad +\frac{\exp\bigparentheses{2\bigparentheses{\frac{1}{2} -s - F_{i^*}(x_0)}^2}}{2\bigparentheses{\frac{1}{2} -s - F_{i^*}(x_0)}^2}\left(\frac{\ceiling{c}K+1}{t}\right)^{\bigparentheses{\frac{1}{2} -s-F_{i^*}}^2c}  + \frac{\exp(2\epsilon_0^2)}{\epsilon^2_0}\bigparentheses{\frac{\ceiling{c}K+1}{t}}^{\epsilon_0^2c}.
\end{align}
Note that  $(\log t) (\frac{1}{t})^{\frac{c-10}{10}} < \frac{10}{e(c-10)}$ for all $t>0$, and  $\bigparentheses{F_j(x_0)-\frac{1}{2} -s}^2c>2$, $\bigparentheses{\frac{1}{2} -s-F_{i^*}(x_0)}^2c>2$, $\epsilon_0^2c>2$ and $\frac{c}{10}>2$ for  $c > \max \left\{20, \frac{2}{\bigparentheses{F_j(x_0)-\frac{1}{2}  -s}^2}, \frac{2}{\bigparentheses{\frac{1}{2} -s - F_{i^*}(x_0)}^2}, \frac{2}{\epsilon_0^2} : j =1, 2, ...K, j \neq i^*\right\}$. Thus, 
we obtain from \eqref{form: medIt2} that, for any  $t\geq B$
\begin{align}\label{form: Aupperbound}
\prob{I_t = j } &\le \frac{c}{t} + \left(\frac{10c(\ceiling{c}K+1)}{e(c-10)} +\frac{\exp\bigparentheses{2\bigparentheses{F_j(x_0)-\frac{1}{2} -s}^2}(\ceiling{c}K+1)}{2\bigparentheses{F_j(x_0)-\frac{1}{2} -s}^2}\right. \nonumber\\
&\qquad\left. +\frac{\exp\bigparentheses{2\bigparentheses{\frac{1}{2} -s - F_{i^*}(x_0)}^2}(\ceiling{c}K+1)}{2\bigparentheses{\frac{1}{2} -s - F_{i^*}(x_0)}^2}  + \frac{\exp(2\epsilon_0^2)(\ceiling{c}K+1)}{\epsilon_0^2} \right)\frac{1}{t^2}.
\end{align}
{\bf Upper-bounding $\bar{R}_T$:}
Based on the upper bound on $\prob{I_t = j } $ in~\eqref{form: Aupperbound}, we now upper-bound the pseudo-regret  $\bar{R}_T$. First note that, taking the expectation, 
\begin{align*}
\bar{R}_T &=\mu^*T- \mathbb{E}[\sum_{t=1}^{T}\mu_{I_t}] = \mathbb{E}[\sum_{\substack{j=1 \\ j\neq i^*}}^{K} \sum_{t =1}^{T} \Delta_j \indicator{I_t =j}] = \sum_{\substack{j=1 \\ j\neq i^*}}^{K} \sum_{t =1}^{T}(\mu_* - \mu_j)\prob{I_t =j},
\end{align*}
which, combined with \eqref{form: Aupperbound}, yields
\begin{align*}
\bar{R}_T &\le A\mu^* +\sum_{j=1, j\neq i^*}^{K}\Delta_jc\log \frac{T}{A} + \sum_{j=1, j\neq i^*}^{K}\Delta_j\left(\frac{10c(\ceiling{c}K+1)}{e(c-10)} +\frac{\exp\bigparentheses{2\bigparentheses{F_j(x_0)-\frac{1}{2} -s}^2}(\ceiling{c}K+1)}{2\bigparentheses{F_j(x_0)-\frac{1}{2} -s}^2}\right. \nonumber\\
&\quad\left. +\frac{\exp\bigparentheses{2\bigparentheses{\frac{1}{2} -s - F_{i^*}(x_0)}^2}(\ceiling{c}K+1)}{2\bigparentheses{\frac{1}{2} -s - F_{i^*}(x_0)}^2}+ \frac{\exp(2\epsilon_0^2)(\ceiling{c}K+1)}{\epsilon_0^2} \right)\frac{1}{A}\\
& \overset{(i)}\le A\mu^* +\sum_{j=1, j\neq i^*}^{K}\Delta_jc\log \frac{T}{A}\\
&\quad + \sum_{j=1, j\neq i^*}^{K}\left(\frac{10c(\ceiling{c}K+1)}{e(c-10)} +\frac{\sqrt{e}(\ceiling{c}K+1)}{2\bigparentheses{F_j(x_0)-\frac{1}{2} -s}^2} +\frac{\sqrt{e}(\ceiling{c}K+1)}{2\bigparentheses{\frac{1}{2} -s - F_{i^*}(x_0)}^2}+ \frac{\sqrt{e}(\ceiling{c}K+1)}{\epsilon_0^2} \right)\Delta_j\frac{1}{A}, 
\end{align*}
where $A= (\ceiling{c}K+1)e$, and ($i$) follows from the fact $\epsilon_0\le \frac{1}{2}$, $\frac{1}{2} -s - F_{i^*}(x_0)< \frac{1}{2}$ and $F_j(x_0)-\frac{1}{2} -s <\frac{1}{2}$.  By exploiting the simple bounds on the constants, we have  

\begin{align*}
\bar{R}_T &\le  c\sum_{j=1, j\neq i^*}^{K}\Delta_j\log T + 2cKe\mu^* + \sum_{j=1, j\neq i^*}^{K}(2+3c)\Delta_j,
\end{align*}
which completes the proof.

\section{Proof of Theorem \ref{th:highprob_meg}}
Our first step is to upper-bound $\indicator{I_t = j}$ for any $j\neq i^*$, and $t\ge\frac{6\ceiling{c}^2K^3}{\delta} := B$. Based on the med-$\epsilon$-greedy algorithm, we obtain

\begin{align}
\indicator{I_t = j} &= \indicator{\textit{explore in round $t$, draw arm j}} + \indicator{\textit{exploit in round $t$, draw arm j}} \nonumber \\
& =\indicator{\textit{explore in round $t$, draw arm j}} + \indicator{\textit{exploit in round $t$},\mathrm{med}_j (t -1) = \mathrm{max}_j \{\mathrm{med}_j (t-1)\}} \nonumber \\
&\le \indicator{\textit{explore in round $t$, draw arm j}} + \indicator{\textit{exploit in round $t$},\mathrm{med}_j (t -1)\ge \mathrm{med}_{i^*} (t-1)}. \nonumber
\end{align}	

Then, we using this fact to upper-bound $T_j(T)$, we obtain
\begin{align}
T_j(T) &= \sum_{t=1}^T \indicator{I_t =j}\nonumber\\
&\le B + \sum_{t=A+1}^T[\indicator{\textit{explore in round $t$, draw arm j}} + \indicator{\textit{exploit in round $t$},\mathrm{med}_j (t -1)\ge \mathrm{med}_{i^*} (t-1)}]. \nonumber\\
\end{align}

We then further obtain 
\begin{align}
R_T &= \sum_{j=1, j\neq i^*}^K \Delta_j T_j(T)\le B\mu^* + \sum_{j=1, j\neq i^*}^K\frac{\Delta_j}{K}\sum_{t=A+1}^T \indicator{\textit{explore in round t}} \nonumber\\
&\qquad + \sum_{j=1, j\neq i^*}^K\Delta_j\sum_{t=B+1}^T\indicator{\textit{exploit in round t}, \mathrm{med}_j(t-1)\ge \mathrm{med}_{i^*}(t-1)}.  \label{for:last}
\end{align}

Now define $T_j^R(t) = \sum_{k=1}^t \indicator{\textit{explore in round k, pull arm }j}$, and $u(t) = \sum_{k=\ceiling{c}K+1}^{t-1} \frac{c}{2k}$. Then we can re-write the third term of \eqref{for:last} in the following way,
\begin{align}
&\indicator{\textit{exploit in round t}, \mathrm{med}_j(t-1)\ge \mathrm{med}_{i^*}(t-1)} \nonumber \\
&\qquad \le \indicator{\mathrm{med}_j(t-1)\ge \mathrm{med}_{i^*}(t-1)}   \le \indicator{\mathrm{med}_j(t-1)\ge x_0}+ \indicator{\mathrm{med}_{i^*}(t-1)\le x_0} \nonumber \\
& \qquad \overset{(i)}\le \indicator{\mathrm{med}_j(t-1)\ge x_0, T_j^R(t-1)> u(t)}+ \indicator{T_j^R(t-1)\le u(t)} \nonumber\\
&\qquad \qquad + \indicator{\mathrm{med}_{i^*}(t-1)\le x_0, T_{i^*}^R(t-1)> u(t)}+ \indicator{T_{i^*}^R(t-1)\le u(t)} \nonumber\\
& \qquad \overset{(ii)}= \sum_{v= \flooring{u(t)}+1}^{t-1} \indicator{\mathrm{med}_j(t-1)\ge x_0, T_j^R(t-1)=v}+ \indicator{T_j^R(t-1)\le u(t)} \nonumber\\
&\qquad \qquad + \sum_{r= \flooring{u(t)}+1}^{t-1}\indicator{\mathrm{med}_{i^*}(t-1)\le x_0, T_{i^*}^R(t-1)=r}+ \indicator{T_{i^*}^R(t-1)\le u(t)} \nonumber\\
& \qquad \overset{(iii)}= \sum_{v= \flooring{u(t)}+1}^{t-1} \indicator{\mathrm{med}_j(t-1)\ge x_0| T_j^R(t-1)=v}+ \indicator{T_j^R(t-1)\le u(t)} \nonumber\\
&\qquad \qquad + \sum_{r= \flooring{u(t)}+1}^{t-1}\indicator{\mathrm{med}_{i^*}(t-1)\le x_0| T_{i^*}^R(t-1)=r}+ \indicator{T_{i^*}^R(t-1)\le u(t)}, \label{for:roundbound}
\end{align}
where $(i)$ follows from the fact $\indicator{\mathcal{U}} \le \indicator{\mathcal{U},\mathcal{T}} + \indicator{\mathcal{T}^c}$ for all events $\mathcal{U}$ and $\mathcal{T}$, $(ii)$ follows from the law of total probability, and $(iii)$ follows from the fact that $\indicator{\mathcal{U}, \mathcal{T}}= \indicator{\mathcal{U}|\mathcal{T}}\indicator{\mathcal{T}}\le \indicator{\mathcal{U}|\mathcal{T}}$ for all event $\mathcal{U}$ and $\mathcal{T}$.

Let $ \{ X_k^j \}_{k=1}^{m} := \{X_{j,k}\}_{k\in Q_j(t), T_j(t) :=m}$ be the first $m$ rewards collected from arm $j$. Then, based on~\eqref{for:roundbound}, we obtain 
\begin{align}
&\indicator{\textit{exploit in round t}, med_j(t-1)\ge med_{i^*}(t-1)} \nonumber\\
& \quad \le \sum_{v= \flooring{u(t)}+1}^{t-1} \indicator{\pquantile{\frac{1}{2}}{\{X_k^{j}\}_{k=1}^v} \ge x_0}+ \indicator{T_j^R(t-1)\le u(t)}\nonumber\\
& \qquad + \sum_{r= \flooring{u(t)}+1}^{t-1}\indicator{\pquantile{\frac{1}{2}}{\{X_k^{i^*}\}_{k=1}^r} \le x_0} + \indicator{T_{i^*}^R(t-1)\le u(t)} \label{for:easydata}.
\end{align}

Substituting \eqref{for:easydata} into \eqref{for:last} yields
\begin{align}
R_T &\le B\mu^* + \sum_{j=1, j\neq i^*}^K\frac{\Delta_j}{K}\sum_{t=B+1}^T \indicator{\textit{explore in round t}}\nonumber\\
& \quad + \sum_{j=1, j\neq i^*}^K\Delta_j\sum_{t=B+1}^T \bigparentheses{\sum_{v= \flooring{u(t)}+1}^{t-1} \indicator{\pquantile{\frac{1}{2}}{\{X_k^{j}\}_{k=1}^v} \ge x_0} + \indicator{T_j^R(t-1)\le u(t)} } \nonumber\\
& \quad +  \sum_{j=1, j\neq i^*}^K\Delta_j\sum_{t=B+1}^T \bigparentheses{\sum_{r= \flooring{u(t)}+1}^{t-1}\indicator{\pquantile{\frac{1}{2}}{\{X_k^{i^*}\}_{k=1}^r} \le x_0} + \indicator{T_{i^*}^R(t-1)\le u(t )}} \nonumber\\
& \overset{(i)}= B\mu^* + \sum_{j=1, j\neq i^*}^K\frac{\Delta_j}{K}\sum_{t=B+1}^T \indicator{\textit{explore in round t}}\nonumber\\
& \quad + \sum_{j=1, j\neq i^*}^K\Delta_j\sum_{t=B+1}^T \indicator{T_j^R(t-1)\le u(t)} + \indicator{T_{i^*}^R(t-1)\le u(t)} \nonumber\\
& \quad + \sum_{j=1, j\neq i^*}^K\Delta_j\sum_{t=B+1}^T \sum_{r= \flooring{u(t)}+1}^{t-1} \bigparentheses{\indicator{\pquantile{\frac{1}{2}}{\{X_k^{j}\}_{k=1}^r} \ge x_0} + \indicator{\pquantile{\frac{1}{2}}{\{X_k^{i^*}\}_{k=1}^r} \le x_0}},\label{for:keyform}
\end{align}
where $(i)$ follows by rearranging the terms.

Define event $\mathcal{K}$, 
\begin{align*}
\mathcal{K} &=\left\{ \sum_{t=B+1}^T\indicator{\textit{explore in round t}} \le 2cK\log(\frac{T}{B})\right\} \bigcap \bigparentheses{\bigcap_{j=1}^K\bigcap_{t=B+1}^T \left\{\indicator{T_j^R(t-1)\le u(t )}\le  \exp\bigparentheses{-\frac{u(t)}{10}}\right\}} \\
&\quad \bigcap \bigcap_{j=1, j\neq i^*}^K\bigcap_{v=u(B)}^T \left\{ \indicator{\pquantile{\frac{1}{2}}{\{X_k^{j}\}_{k=1}^v} \ge x_0} \le \exp\bigparentheses{-v[F_j(x_0) - \frac{1}{2} - s]^2} \right\} \\
&\quad \bigcap \bigcap_{r=u(A)}^T \left\{\indicator{\pquantile{\frac{1}{2}}{\{X_k^{i^*}\}_{k=1}^r} \ge x_0} \le \exp\bigparentheses{-v[\frac{1}{2} - s - F_{i^*}(x_0)]^2} \right\}.
\end{align*}

Under the event $\mathcal{K}$, we obtain 
\begin{align}
R_T &\le B\mu^* + \sum_{j=1, j\neq i^*}^K 2c\Delta_j \log \frac{T}{B}+\underbrace{\sum_{j=1, j\neq i^*}^K2\Delta_j\sum_{t=B+1}^T \exp \bigparentheses{-\frac{u(t)}{10}}}_{(a)} \nonumber\\
& \quad + \underbrace{\sum_{j=1, j\neq i^*}^K\Delta_j\sum_{t=B+1}^T \sum_{r= \flooring{u(t)}+1}^{t-1}  \exp\bigparentheses{-r[F_j(x_0) - \frac{1}{2} - s]^2} + \exp\bigparentheses{-r[\frac{1}{2} - s - F_{i^*}(x_0)]^2}}_{(b)}. \label{for:midRT}
\end{align}

For the term $(a)$ in \eqref{for:midRT}, we can upper-bound it by
\begin{align*}
\sum_{j=1, j\neq i^*}^K2\Delta_j\sum_{t=B+1}^T \exp \bigparentheses{-\frac{u(t)}{10}}&\le \sum_{j=1, j\neq i^*}^K2\Delta_j\sum_{t=B+1}^T\exp \bigparentheses{-\frac{c}{20}\log \frac{t}{\ceiling{c}K+1}}\\
& \overset{(i)}=\sum_{j=1, j\neq i^*}^K2\Delta_j\sum_{t=B+1}^T \bigparentheses{\frac{\ceiling{c}K+1}{t}}^\frac{c}{20}  \\
&\overset{(ii)}\le \sum_{j=1, j\neq i^*}^K2\Delta_j\sum_{t=B+1}^T \bigparentheses{\frac{\ceiling{c}K+1}{t}}^2\le \sum_{j=1, j\neq i^*}^K \frac{2\Delta_j(\ceiling{c}K+1)^2}{B},
\end{align*}
where $(i)$ follows from the fact that $u(t)\ge \frac{c}{2} \frac{t}{\ceiling{c}K+1}$ and $(ii)$ follows from that $c\ge 40$.

For the term $(b)$ in \eqref{for:midRT}, we can upper-bound it by 
\begin{align*}
&\sum_{j=1, j\neq i^*}^K\Delta_j\sum_{t=B+1}^T \sum_{r= \flooring{u(t)}+1}^{t-1}  \exp\bigparentheses{-r[F_j(x_0) - \frac{1}{2} - s]^2} + \exp\bigparentheses{-r[\frac{1}{2} - s - F_{i^*}(x_0)]^2} \\
&\quad \overset{(i)}\le \sum_{j=1, j\neq i^*}^K\Delta_j\sum_{t=B+1}^T \frac{1}{[F_j(x_0) - \frac{1}{2} - s]^2} \exp\bigparentheses{-[F_j(x_0) - \frac{1}{2} - s]^2(u(t)-1)} \\
& \qquad \sum_{j=1, j\neq i^*}^K\Delta_j\sum_{t=B+1}^T \frac{1}{[\frac{1}{2} - s - F_{i^*}(x_0)]^2} \exp\bigparentheses{-[\frac{1}{2} - s - F_{i^*}(x_0)]^2(u(t)-1)}\\
&\quad\overset{(ii)} \le \sum_{j=1, j\neq i^*}^K c\Delta_j \sum_{t=B+1}^T \exp \bigparentheses{-\frac{4u(t)}{c}} \overset{(iii)}\le \sum_{j=1, j\neq i^*}^K c\Delta_j \sum_{t=B+1}^T \left(\frac{\ceiling{c}K+1}{t}\right)^2\\
&\quad\le \sum_{j=1, j\neq i^*}^K c\Delta_j(\ceiling{c}K+1)^2\frac{1}{B},
\end{align*}
where $(i)$ follows from the fact that $\sum_{t=x+1}^T \exp(-Zt)\le \frac{\exp(-Zx)}{Z}$ holds for all $x\in \mathbb{N}$, $T\in\mathbb{N}$, and $Z\in \mathbb{R}^+$, $(ii)$ follows from the fact that $c\ge \frac{4}{[\frac{1}{2} - s - F_{i^*}(x_0)]^2}, \frac{4}{[F_j(x_0) - \frac{1}{2} - s]^2}$ and $(iii)$ comes from the fact that $u(t)\ge \frac{c}{2} \log\frac{t}{\ceiling{c}K+1}$ .

Hence, we derive the upper-bound of $R_T$ as follows.
\begin{align*}
R_T &\le B\mu^* + \sum_{j=1, j\neq i^*}^K 2c\Delta_j \log \frac{T}{B}+ \sum_{j=1, j\neq i^*}^K \frac{2\Delta_j(\ceiling{c}K+1)^2}{B} + \sum_{j=1, j\neq i^*}^K c\Delta_j(\ceiling{c}K+1)^2 \frac{1}{B} \\
&\overset{(i)}\le   \frac{6\ceiling{c}^2K^3}{\delta}\mu^* + \sum_{j=1, j\neq i^*}^K2c\Delta_j \log T + \sum_{j=1, j\neq i^*}^K 2c\Delta_j,
\end{align*}
where $(i)$ follows from our assumptions about parameters.

The final step is to derive an lower-bound of probability of $\mathcal{K}$,

In our paper, Note that lemma 1 shows that if $N = \ceiling{\frac{1}{2\epsilon_0^2} \log \frac{K}{\epsilon_0^2\delta}}+1$, event $\mathcal{E} = \{ s_j(m_j) \le \rho + \epsilon_0 : m_j \ge N, j = 1,... K \}$ occurs with the probability at least $1- \frac{\delta}{2}$. Then we have

\begin{align}
\prob{\mathcal{K}^c} &= \prob{\mathcal{K}^c, \mathcal{E}} + \prob{\mathcal{K}^c, \mathcal{E}^c}\le \prob{\mathcal{K}^c| \mathcal{E}} + \prob{\mathcal{E}^c}\le \prob{\mathcal{K}^c| \mathcal{E}} + \frac{\delta}{2}.\label{for:prob}
\end{align}
Due to the fact $\prob{\mathcal{U}\cup \mathcal{T}} \le \prob{\mathcal{U}}+ \prob{\mathcal{T}}$ for any events $\mathcal{U}$ and $\mathcal{T}$, we obtain
\begin{align*}
\prob{\mathcal{K}^c | \mathcal{E}}&\le \sum_{j=1}^K\sum_{t=B+1}^T \prob{\indicator{T_j^R(t-1)\le u(t )}\ge  \exp\bigparentheses{-\frac{u(t)}{10}}}  + \prob{\sum_{t=B+1}^T\indicator{\textit{explore in round t}} \ge 2cK\log(\frac{T}{B})} \\
&\quad + \sum_{j=1, j\neq i^*}^K\sum_{v=u(B)}^T \prob{\indicator{\pquantile{\frac{1}{2}}{\{X_k^{j}\}_{k=1}^v} \ge x_0} \ge \exp\bigparentheses{-v[F_j(x_0) - \frac{1}{2} - s]^2} |  \mathcal{E}} \\
&\quad + \sum_{r=u(B)}^T \prob{\indicator{\pquantile{\frac{1}{2}}{\{X_k^{i^*}\}_{k=1}^r} \le x_0} \ge \exp\bigparentheses{-v[\frac{1}{2} - s - F_{i^*}(x_0)]^2} |  \mathcal{E}},
\end{align*}
where, $T_j^R$ and $\indicator{\textit{explore in round t}}$ are independent with event $\mathcal{E}$. Thus, the event $\mathcal{E}$ can be removed from the conditioning.

Applying the variant of Bernstein inequality~\cite[Fact 2]{auer2002finite} to $T_j^R(t-1)$ yields,
\begin{align*}
\expectation{\indicator{T_j^R(t-1)\le u(t )}}= \prob{T_j^R(t-1)\le u(t )} \le \exp\bigparentheses{-\frac{u(t)}{5}}.
\end{align*}
Moreover, applying Markov inequality, we obtain
\begin{align*}
\prob{\indicator{T_j^R(t-1)\le u(t )}\ge  \exp\bigparentheses{-\frac{u(t)}{10}}} &\le\frac{\expectation{\indicator{T_j^R(t-1)\le u(t )}}}{\exp\bigparentheses{-\frac{u(t)}{10}}}\exp\bigparentheses {-\frac{c}{20} \log \frac{t}{\ceiling{c}K}}&\le \bigparentheses{\frac{\ceiling{c}K}{t}}^{\frac{c}{20}}.
\end{align*}

Using the fact that  $c\ge 40$, we obtain
\begin{align}
\sum_{j=1}^K\sum_{t=B+1}^T \prob{\indicator{T_j^R(t-1)\le u(t )}\ge  \exp\bigparentheses{-\frac{u(t)}{10}}} & \le \sum_{j=1}^K\sum_{t=B+1}^T \bigparentheses{\frac{\ceiling{c}K}{t}}^2\le K \frac{(\ceiling{c}K)^2}{B}\overset{(i)}\le \frac{\delta}{6}, \label{for:first}
\end{align}
where $(i)$ follows from the fact $B \ge \frac{6\ceiling{c}^2K^3}{\delta}$.

Applying the variant of Bernstein inequality~\cite[Fact 2]{auer2002finite} to $\sum_{t=A+1}^T\indicator{\textit{explore in round t}}$, we obtain
\begin{align}
\prob{\sum_{t=B+1}^T\indicator{\textit{explore in round t}}\ge 2cK\log(\frac{T}{B})}&\le \exp\bigparentheses{\frac{cK}{3} \log \frac{B}{T}} = \bigparentheses{\frac{B}{T}}^{\frac{cK}{3}} \overset{(i)}\le \frac{\delta}{6} \label{for:second}
\end{align}
where $(i)$ holds for $T\ge B (\frac{6}{\delta})^{\frac{cK}{3}}$.

Since $u(B) \ge \frac{c}{2}\log\frac{B}{\ceiling{c}K+1} \ge N$, for all $v\ge u(B)\ge N$, we have 
\begin{align*}
\expectation{\indicator{\pquantile{\frac{1}{2}}{\{X_k^{j}\}_{k=1}^v} \ge x_0}| \mathcal{E}} \le \exp\bigparentheses{-2v[F_j(x_0) - \frac{1}{2} - s]^2},
\end{align*}
and 
\begin{align*}
\expectation{\indicator{\pquantile{\frac{1}{2}}{\{X_k^{i^*}\}_{k=1}^v} \le x_0}| \mathcal{E}} \le \exp\bigparentheses{-2v[\frac{1}{2} - s- F_{i^*}(x_0)]^2},
\end{align*}
Thus using Markov inequality we obtain,
\begin{align}
&\sum_{j=1, j\neq i^*}^K\sum_{v=u(B)}^T \prob{\indicator{\pquantile{\frac{1}{2}}{\{X_k^{j}\}_{k=1}^v} \ge x_0} \ge \exp\bigparentheses{-v[F_j(x_0) - \frac{1}{2} - s]^2} |  \mathcal{E}} \nonumber \\	
&\quad + \sum_{r=u(B)}^T \prob{\indicator{\pquantile{\frac{1}{2}}{\{X_k^{i^*}\}_{k=1}^r} \le x_0} \ge \exp\bigparentheses{-v[\frac{1}{2} - s - F_{i^*}(x_0)]^2} |  \mathcal{E}} \nonumber \\
& \quad\le \sum_{j=1, j\neq i^*}^K\sum_{v=u(B)}^T \frac{\expectation{\indicator{\pquantile{\frac{1}{2}}{\{X_k^{j}\}_{k=1}^v} \ge x_0}| \mathcal{E}}}{\exp\bigparentheses{-v[F_j(x_0) - \frac{1}{2} - s]^2}}\nonumber\\
&\qquad + \sum_{v=u(B)}^T \frac{\expectation{\indicator{\pquantile{\frac{1}{2}}{\{X_k^{i^*}\}_{k=1}^v} \le x_0}| \mathcal{E}}}{\exp\bigparentheses{-v[\frac{1}{2} - s- F_{i^*}(x_0)]^2}}\nonumber\\
& \quad \le \sum_{j=1, j\neq i^*}^K\sum_{v=u(B)}^T  \exp\bigparentheses{-v[F_j(x_0) - \frac{1}{2} - s]^2}+  \sum_{v=u(B)}^T \exp\bigparentheses{-v[\frac{1}{2} - s- F_{i^*}(x_0)]^2} \nonumber \\
& \quad \overset{(i)}\le \sum_{j=1}^K\sum_{v=u(B)}^T \exp\bigparentheses{-\frac{2v}{c}}\le K \frac{c}{2} \exp\bigparentheses{-\frac{2u(B)}{c}} \le \frac{cK (\ceiling{c}K+1)}{2B} < \frac{c^2K^2}{B}\nonumber \\
&\quad\overset{(ii)}\le \frac{\delta}{6}, \label{for:third}
\end{align}
where $(i)$ follows from the fact that $c \ge \frac{2}{[\frac{1}{2} - s- F_{i^*}(x_0)]^2}, \frac{2}{[F_j(x_0) - \frac{1}{2} - s]^2}$, and $(ii)$ follows from the fact $B\ge \frac{6\ceiling{c}^2K^3}{\delta}$.

Substituting \eqref{for:first}, \eqref{for:second} and \eqref{for:third} into \eqref{for:prob}, we have
$$\prob{\mathcal{K}^c}\le \delta,$$

which completes the proof.

\section{Proof of \Cref{cor: gau_meg}}

First, we show that Gaussian distributions meet Assumption \ref{ass:meg}, and specify the constant $x_0$, and $s$. Then, we apply \Cref{th:highprob_meg} and \Cref{th: avg_meg} to complete our proof.

Let $s =\Phi(\frac{\Delta_{min}}{4\sigma}) - \frac{1}{2}$, and $x_0 = \mu^* - \frac{\Delta_{min}}{2}$. Then, 
$$F_{i*}(x_0) = \Phi\left(\frac{x_0- \mu^*}{\sigma}\right)  = \Phi\left(-\frac{\Delta_{min}}{2\sigma} \right)=1 - \Phi\left(\frac{\Delta_{min}}{2\sigma} \right) .$$

Clearly, $F_{i*}(x_0) <\frac{1}{2} -s$. Moreover, for any $j\neq i^*$,
$$F_j(x_0) = \Phi\left(\frac{x_0- \mu_j}{\sigma}\right) = \Phi\left(\frac{\Delta_j -\frac{\Delta_{min}}{2}}{\sigma}\right)\ge \Phi\left(\frac{\Delta_{min}}{2\sigma}\right).$$ 

It is also clear that $F_j(x_0) > s + \frac{1}{2}$. Hence, the given Gaussian distributions meet Assumption\ref{ass:meg}.

For all $j \neq i^*$, $$\frac{1}{(F_j(x_0) - \frac{1}{2} - s)^2} \le \frac{1}{\Phi(\frac{\Delta_{min}}{2\sigma}) -\Phi(\frac{\Delta_{min}}{4\sigma})}.$$

Therefore, if $\rho<\Phi(\frac{\Delta_{min}}{4\sigma}) - \frac{1}{2}$,
$$c > \max\left\{10, \frac{1}{\left(\Phi(\frac{\Delta_{min}}{2\sigma})- \Phi(\frac{\Delta_{min}}{4\sigma})\right)^2}, \frac{1}{(\Phi(\frac{\Delta_{min}}{4\sigma}) - \frac{1}{2}- \rho)^2}\right\}.$$
Applying \Cref{th:highprob_meg} and \Cref{th: avg_meg} completes the proof.

\end{document}